\documentclass[11pt]{article}       
\usepackage{geometry}               
\usepackage{graphicx}
\usepackage{caption}
\usepackage{subcaption}
\usepackage{enumerate}
\usepackage{hyperref}
\usepackage{upgreek}
\usepackage{todonotes}
\usepackage{cite}

\usepackage{amsmath} % assumes amsmath package installed
\usepackage{amssymb}  % assumes amsmath package installed
\usepackage{amsthm}  % assumes amsmath package installed
\usepackage{mathrsfs}
\usepackage{mathtools}
\usepackage{bm}
\usepackage{cite}
\usepackage{lipsum}
\usepackage[linesnumbered,ruled,vlined]{algorithm2e}
\usepackage{color}
\usepackage{array}
\usepackage{multirow}
\newtheorem{proposition}{Proposition}

\newtheorem{theorem}{Theorem}

\newtheorem{cor}{Corollary}
\newcolumntype{M}[1]{>{\centering\arraybackslash}m{#1}}

\newcommand{\dist}{\mathrm{dist}}

\newcommand{\dcvar}{\mathrm{DR\mbox{-}CVaR}}

\DeclareMathOperator*{\argmin}{arg\,min}

\title{Distributionally Robust Risk Map for Learning-Based Motion Planning and Control: A Semidefinite Programming Approach\thanks{This work was supported in part by  the Creative-Pioneering Researchers Program through SNU, the National Research Foundation of Korea funded by the MSIT(2020R1C1C1009766), the Information and Communications Technology Planning and Evaluation (IITP) grant funded by MSIT(2020-0-00857), and Samsung Electronics.}}

\author{Astghik~Hakobyan
\and
 Insoon Yang\thanks{A. Hakobyan, and I. Yang are with the Department of Electrical and Computer Engineering, Automation and Systems Research Institute,  Seoul National University, Seoul 08826, Korea, 
        {\tt\small \{astghikhakobyan, insoonyang\}@snu.ac.kr}} %({\it Corresponding author: Insoon Yang.})} 
        }

\date{}

\begin{document}
\maketitle
% !BIB program = 
\pagestyle{myheadings}
\thispagestyle{plain}

\begin{abstract}
This paper proposes a novel safety specification tool, called the \emph{distributionally robust risk map} (DR-risk map), for a mobile robot operating in a learning-enabled environment.  
Given the robot's position, the map aims to reliably assess the conditional value-at-risk (CVaR) of collision with obstacles whose movements are inferred by Gaussian process regression (GPR). 
Unfortunately, the inferred distribution is subject to errors, making it difficult to accurately evaluate the CVaR of collision. 
To overcome this challenge, this tool measures the risk under the worst-case distribution in a so-called \emph{ambiguity set} that characterizes allowable distribution errors. 
To resolve the infinite-dimensionality issue inherent in the construction of the DR-risk map, we derive a tractable semidefinite programming formulation that provides an upper bound of the risk, exploiting techniques from modern distributionally robust optimization. 
As a concrete application for motion planning, a distributionally robust RRT* algorithm is considered using the risk map that addresses distribution errors caused by GPR. 
Furthermore, a motion control method is devised using the DR-risk map in a learning-based model predictive control (MPC) formulation. 
In particular, a neural network approximation of the risk map is proposed to reduce the computational cost in solving the MPC problem.
 The performance and utility of the proposed risk map are demonstrated through simulation studies that show its ability to ensure the safety of mobile robots despite learning errors.
\end{abstract}

\section{Introduction}

Ensuring safety in motion planning and control critically depends on the quality of information about the possibly uncertain environment in which a robot operates. 
For example, a mobile robot may use sensor measurements
to take into account the uncertain behavior of other robots, human agents, or obstacles for collision avoidance. 
With advances in machine learning, sensing, and computing technologies,
the adoption of state-of-the-art learning techniques is rapidly growing 
for a robot to infer the evolution of its environment. 
Unfortunately, the accuracy of inference is often poor since it is subject to the quality of the observations, statistical models, and learning methods. Using inaccurately learned information in the robot's decision-making may induce unwanted behaviors, and, in particular, may lead to a collision. 
This work aims to develop a safety risk specification tool that is robust against distribution errors in learned information about moving obstacles and is thus useful for ensuring safety in learning-based motion planning and control.

Safety specification tools for systems with learning-enabled components
can be categorized into two classes. 
The first class concerns the safety of learning-enabled robots, while the second class considers learning-enabled environments. 
The tools in the first class use or learn reachable sets~\cite{gillula2012guaranteed, fisac2018general, shao2021reachability}, Lyapunov functions~\cite{richards2018lyapunov, taylor2019episodic}, or control barrier functions~\cite{wang2018safe, cheng2019end, taylor2020learning} as a certificate for safety when the system dynamics of robots are unknown.
The literature on the second class is relatively sparse. 
Existing methods to handle learning-enabled environments use chance constraints~\cite{du2011robot}, logistic functions~\cite{richter2014high}, collision detection via Monte Carlo sampling~\cite{eidehall2008statistical}, and detection of conflicts between intention and expectation~\cite{lefevre2012risk}, among others. 
Our method belongs to the second class and,
departing from previously used tools, we take a \emph{conditional value-at-risk} (CVaR) approach since CVaR is capable of distinguishing rare tail events~\cite{Rockafellar2002a}. 

This work is also related to learning-based motion planning and control, which are the main applications of our safety specification tool. 
The following two cases are considered in the literature: $(i)$ learning the system dynamics of robots, and $(ii)$ learning the environment.
The first case is the most well-studied direction, which is based on RRT*~\cite{bry2011rapidly, liu2014incremental}, model-predictive control~\cite{Aswani2013, Ostafew2016, Williams2018, hewing2019cautious}, and model-based reinforcement learning (RL)~\cite{Hester2012, venkatraman2016improved, Polydoros2017}, among other methods. These tools employ various learning or inference techniques to update unknown system model parameters that are, in turn, used to improve control actions or policies. On the other hand, the methods in the second class emphasize learning the environment.
In particular, for learning the behavior (or intention) of obstacles or other vehicles, 
several methods have been proposed that use inverse RL~\cite{Kuderer2015, Herman2015, Wulfmeier2017}, imitation learning~\cite{Kuefler2017, codevilla2018end}, and Gaussian mixture models~\cite{chernova2007confidence, Lenz2017}, among others.
The learned information about environments can then be used in probabilistic or robust motion planning and control algorithms~\cite{fulgenzi2009probabilistic, luber2012socially, aoude2013probabilistically,  pereira2013risk, chi2017risk, brito2019model}.
Our method is classified as the second type since it uses the
 learned information about the movement of obstacles.
However, unlike the previous approaches, we emphasize the importance of decision-making that is robust against potential errors caused by learning the environment.
For this, we take a distributionally robust optimization (DRO) approach~\cite{esfahani2018data, gao2016distributionally, kuhn2019wasserstein} to address  errors in learned information regarding the motion of obstacles.

In this work, we propose a novel safety specification tool, which we call the \emph{distributionally robust risk map} (DR-risk map).
It is a spatially varying function that specifies the safety risk in a way that is robust against errors in learning or prediction results about the obstacles' locations. 
Specifically, the obstacles' future trajectories are assumed to be inferred using GPR based on the current and past observations. 
However, the predicted probability distribution of the obstacles' locations is subject to errors, making it difficult to accurately evaluate the risk of collision. 
To resolve this issue, our method evaluates the risk under the worst-case distribution in a so-called \emph{ambiguity set}.
Thus, the robot's decision made using the DR-risk map will generate a safe behavior even when the true distribution deviates from the learned one within the ambiguity set.
Unfortunately, the computation of DR-risk is challenging since it involves the infinite-dimensional optimization problem over the ambiguity set of probability distributions.

The main contributions of this work are threefold.
First, we propose a tractable semidefinite programming (SDP) formulation that provides an upper bound of the DR-risk map. 
The SDP approach, which exploits techniques from DRO,  alleviates the infinite-dimensionality issue inherent in the DR-risk map.  Further, we provide its dual formulation, which has fewer generalized equalities.
Second, we demonstrate the utility of the DR-risk map in learning-based motion planning. 
A distributionally robust RRT* algorithm is proposed to use the risk map for generating a safe path despite the learning errors caused by GPR. 
Third, we devise a motion control tool that employs the neural network (NN) approximation of the DR-risk map. 
Our method uses MPC with risk constraints that can be evaluated by solving SDPs. 
To avoid solving the SDPs in real-time, we propose approximating the DR-risk map as an NN, which is then embedded in the MPC problem. 
Our NN approximation has the salient feature that the same NN can be used to approximate the DR-risk map for any time and any obstacles
since the dependence is encoded in the input information. 
The performance and utility of the DR-risk map are demonstrated through 
simulation studies for autonomous vehicles and service robots. 
The results of our experiments show that our motion planning and control tools successfully ensure safety even in the presence of distribution errors caused by GPR.

This paper has been significantly expanded from its preliminary conference version~\cite{Hakobyan2020}. The DR-risk map is formally defined, and its SDP approximation is proposed in this paper.
In particular, the construction of DRO is simplified without sampling from the distribution obtained by GPR. 
Furthermore, a motion planning algorithm is proposed using the DR-risk map, unlike the conference version, which focuses on motion control. 
Last but not least, the NN approximation of risk constraints in motion control is newly considered in this paper.

The remainder of the paper is organized as follows. In Section~\ref{sec:prel}, we introduce the problem setup and the GPR approach to learning the future trajectories of obstacles. In Section~\ref{sec:risk_map}, we define the DR-risk map and present its tractable reformulation as an SDP. In Section~\ref{sec:LBDRMP}, we propose a motion planning algorithm using the DR-risk map to address errors caused by GPR. In Section~\ref{sec:LBDRMC}, the risk map is approximated by an NN and applied to an MPC problem for motion control. Finally, in Section~\ref{sec:result}, we present the application of our risk map to motion planning and control problems through simulations in various environments.

\section{Preliminaries}\label{sec:prel}

\subsection{Notation}

We let  $\mathbb{R}_+^n$ denote the set of real vectors with non-negative entries.
The cone of symmetric matrices in $\mathbb{R}^{n \times n}$ is denoted by $\mathbb{S}^n$, while $\mathbb{S}_+^n$ denotes the cone of symmetric positive semidefinite matrices in $\mathbb{S}^n$.
For $A, B \in \mathbb{S}^n$, the notation $A \preceq B$ represents that $B - A \in \mathbb{S}_+^n$. 
We let $\mathcal{P}(\Xi)$ denote the set of Borel probability measures with support $\Xi$. 
The expected value of random variable $X$ with probability measure $\mathrm{P} \in \mathcal{P}(\Xi)$ is denoted by $\mathbb{E}^{\mathrm{P}}[X]$.

\subsection{Mobile Robot and Obstacles}~\label{sec:MRO}

In this work, we consider a mobile robot modeled by the following discrete-time system:
\begin{equation}
\begin{split}
x_r (t+1)&=f(x_r (t),u_r (t))\label{sysmod}\\
y_r(t)&=C x_r (t),
\end{split}
\end{equation}
where $x_r (t)\in \mathbb{R}^{n_x}$, $u_r (t)\in\mathbb{R}^{n_u}$ and $y_r (t)\in   \mathbb{R}^{n_y}$ are the robot's state,  input, and output, respectively, where the subscript `$r$' represents `robot'. The system output is defined as the Cartesian coordinates of the robot's center of mass (CoM).

The robot navigates a cluttered environment with $L$ moving obstacles, e.g., other robotic vehicles. 
The motion of the $\ell$th obstacle is described by the following discrete-time system for $\ell = 1, \ldots, L$:
\begin{align}
x_{o}^{\ell}(t+1)&=\phi^\ell (x_{o}^{\ell}(t),u_o^{\ell}(t))\\
y^{\ell}_o(t)&=C_o^\ell  x_o^{\ell}(t),
\label{obs_model}
\end{align}
where $x^{\ell}_o(t)\in\mathbb{R}^{n_{x}^\ell}$ and $u_o^{\ell}(t)\in\mathbb{R}^{n_{u}^\ell}$ are the obstacle's state and input, respectively.
The subscript `$o$' represents `obstacle'.
 The output $y^{\ell}_o(t)\in\mathbb{R}^{n_{y}}$ is the Cartesian coordinates of the obstacle's CoM and has the same dimension as the robot's output $y_r (t)$.
Here, $\phi^\ell$ is a possibly unknown (nonlinear) function. 
 In practice, $\phi^\ell$ can be replaced with its parametric approximation $\phi_w^\ell$, for example, using NNs, and the parameters $w$ can be estimated using training data. See Appendix~\ref{appendix1} for an example. 
 For ease of exposition, we assume that $\phi^\ell$ or its parametric approximation is given.

For safety, our robot should navigate within a safe region, which is determined by the obstacles' behaviors. The safe region for each obstacle can be defined as the region outside the open ball centered at the obstacle's CoM with \emph{safe distance} $r_{\ell} > 0$:
\begin{equation}
\mathcal{Y}^{\ell}(t):= \big \{y_r(t)\in\mathbb{R}^{n_{y}} \mid \mathrm{dist}(y_r(t),y^{\ell}_o(t)) \geq r_{\ell} \big \},
\label{safereg}
\end{equation}
where $\dist(y_r(t),y^{\ell}_o(t))$ is the Euclidean distance between the robot's CoM and the obstacle's CoM, defined by
\[
\mathrm{dist}(y_r(t),y^{\ell}_o(t)) :=\|y_r(t)-y^{\ell}_o(t)\|_2.
\]

\begin{figure}
\centering
\includegraphics[width=0.35\linewidth]{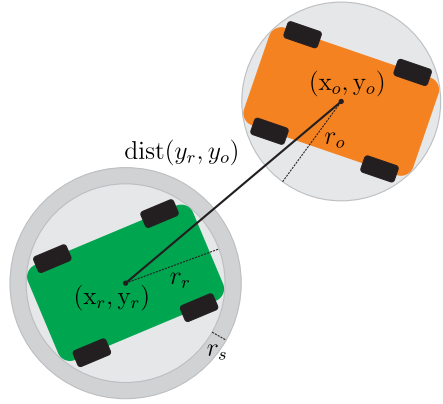}
\caption{The car-like robot (green) is centered at $y_r := (\mathrm{x}_r,\mathrm{y}_r)$, while the obstacle (orange) is centered at $y_o := (\mathrm{x}_o,\mathrm{y}_o)$.  The smallest balls enclosing the robot and the obstacle have radii $r_r$ and $r_o$, respectively. 
With margin $r_s$, the safe distance
$r_\ell$ can be chosen as $r_r + r_o + r_s$.} \label{fig:Safety_Def}
\end{figure}

An example of such a configuration is shown in Fig.~\ref{fig:Safety_Def}, where a car-like robot (green) should navigate to avoid a car-like obstacle  (orange). Both the robot and the obstacle are approximated by the smallest balls enclosing them with radii $r_r$ and $r^{\ell}_o$, respectively. Using an additional safety margin $r_s$, the  distance between the CoMs of the robot and the obstacle should be no smaller than the sum of all radii:
\[
r_{\ell}=r_r+r^{\ell}_o+r_s.
\]

Having $L$ surrounding obstacles, the safe region with respect to all obstacles is defined as the intersection of all the safe regions $\mathcal{Y}^\ell (t)$:
\[
\mathcal{Y}(t):= \bigcap_{\ell=1}^{L} \mathcal{Y}^\ell (t).
\]
Note that the safe region is time-varying.

\subsection{Learning the Motion of Obstacles via
Gaussian Process Regression}

Even though the dynamics $\phi^\ell$ of obstacles are assumed to be known or estimated using some function approximators, 
the actions taken by the obstacles are unknown; thus, our robot has no information about the obstacles' future behaviors. Furthermore, even if the actions were known, the resulting trajectories might include some inaccuracies since $\phi^\ell$ might not accurately describe the real motion of the obstacles. To take such uncertainties into account, the observations made by the robot can be useful for inferring (or learning) the obstacles' movements.

In this study, we use GPR, which is one of the most popular non-parametric methods for learning a probability distribution over all possible values of a function~\cite{rasmussen2003gaussian}. Ideally, GPR can be used to directly infer the future state of obstacle $\ell$ given the current state information. However, leveraging some information about the system dynamics can significantly increase the accuracy of predictions and reduce the size of the required training data. Hence, in this work, we aim to learn the function $\psi^\ell$ that corresponds to the control action of the obstacle $\ell$ given its state information and use it in conjunction with the obstacle dynamics $\phi^{\ell}$ to predict the future trajectory. For ease of exposition, we suppress the superscript $\ell$.

GPR is performed on a training dataset, which is constructed from previous observations about the obstacle's state and action. In particular,  at stage $t$,
the training input data is chosen as $\hat{\mathrm{x}}=\{x_o(t-1),x_o(t-2),\dots,x_o(t-M)\}$ with the corresponding training output data $\hat{\mathrm{y}}=\{u_o(t-1),u_o(t-2),\dots,u_o(t-M)\}$, where $M$ is the number of observations. Since observations are imperfect, we assume that for the $i$th observation
\[
\hat{\mathrm{y}}^{i}=\psi(\hat{\mathrm{x}}^{i})+v,
\]
where $v$ is an i.i.d. zero-mean Gaussian noise with  covariance $\Sigma^v=\mathrm{diag}([\sigma_{v,1}^2\;\sigma_{v,2}^2,\dots,\sigma_{v,n_{u}}^2])$.
Assuming that each control action has independent entries, the GPR dataset for the $j$th dimension of control action is constructed as
\[
\mathcal{D}_j=\big\{\big(\hat{\mathrm{x}}^{i},\hat{\mathrm{y}}_j^{i}\big),\;i=1,\dots,M\big\}
\]
for $j = 1, \ldots, n_u$. 

In GPR,  each dimension of $\psi(\cdot)$ has a Gaussian prior distribution with mean function $m_j(x)$ and kernel $k_j(x,x')$. In this paper, we use a zero-mean prior   with the following radial basis function (RBF) kernel:
\[
k_j(x,x')=\sigma_{f,j}^2\exp\Big[-\frac{1}{2}(x-x')^\top L_j^{-1}(x-x')\Big],
\]
where $L_j$ is a diagonal length scale matrix and   $\sigma_{f,j}^2$ is a signal variance.
The prior on the noisy observations is a normal distribution with mean function $m_j(\hat{\mathrm{x}}^{i})$ and covariance function $K_j(\hat{\mathrm{x}},\hat{\mathrm{x}})+\sigma_{v,j}^2I$, where $K_j(\hat{\mathrm{x}},\hat{\mathrm{x}})$ denotes the $M\times M$ covariance matrix of training input data, i.e., $K_j^{(l,k)}(\hat{\mathrm{x}},\hat{\mathrm{x}})=k_j(\hat{\mathrm{x}}^{(l)},\hat{\mathrm{x}}^{(k)})$.

For a new arbitrary test point $\mathbf{x}$, the posterior distribution of the $j$th output entry  is also Gaussian.
Its mean and covariance are calculated as follows:
\begin{align}
&\mu_u^j(\mathbf{x})= m_j(\mathbf{x})+K_j(\mathbf{x},\hat{\mathrm{x}})(K_j(\hat{\mathrm{x}},\hat{\mathrm{x}})+\sigma_{v,j}^2I)^{-1}(\hat{\mathrm{y}}_j-m_j(\hat{\mathrm{x}}))\label{mean}\\
&\Sigma_u^j(\mathbf{x}) = k_j(\mathbf{x},\mathbf{x})-K_j(\mathbf{x},\hat{\mathrm{x}})(K_j(\hat{\mathrm{x}},\hat{\mathrm{x}})+\sigma_{v,j}^2I)^{-1}K_j(\hat{\mathrm{x}},\mathbf{x}).\label{cov}
\end{align}
The resulting GP approximation of $\psi(\mathbf{x})$ is given by
\[
\psi(\mathbf{x})\sim\mathcal{N}(\mu_u(\mathbf{x}),\Sigma_u(\mathbf{x})),
\]
where $\mu_u(\mathbf{x})=[\mu_u^1(\mathbf{x}),\mu_u^2(\mathbf{x}),\dots, \mu_u^{n_u}(\mathbf{x})]^\top$ and $\Sigma_u(\mathbf{x})=\mathrm{diag}([\Sigma_u^1(\mathbf{x}),\Sigma_u^{2}(\mathbf{x}),\dots,\Sigma_u^{n_u}(\mathbf{x})])$.

The GP approximation of the obstacle's  input is computed given its current state.
 At stage $t$, for each prediction time $t+k$, where $k = 1, \dots, K$ and $K$ is the prediction horizon, the obstacle's state and action are approximated as a joint Gaussian distribution of the form
\[
\begin{bmatrix}
x_o(t+k)\\ u_o(t+k)
\end{bmatrix}\sim\mathcal{N}\bigg(
\begin{bmatrix}
\tilde{\mu}_x^{t,k}\\
\tilde{\mu}_u^{t,k} 
\end{bmatrix},\begin{bmatrix}\tilde{\Sigma}_x^{t,k} & \tilde{\Sigma}_{xu}^{t,k}\\
\tilde{\Sigma}_{ux}^{t,k} & \tilde{\Sigma}_u^{t,k}\end{bmatrix}\bigg),
\] 
where the superscript $(t,k)$ denotes the $(t+k)$th prediction at stage $t$. 
By the first-order Taylor expansion of \eqref{mean} and \eqref{cov}, the mean and covariance  information about $u_o(t+k)$ is obtained as
\begin{equation}
\begin{split}
\tilde{\mu}_{u}^{t,k}&=\mu_{u}(\tilde{\mu}_{x}^{t,k})\\
\tilde{\Sigma}_{u}^{t,k}&=\Sigma_{u}(\tilde{\mu}_{x}^{t,k})+\nabla\mu_{u}(\tilde{\mu}_{x}^{t,k})\tilde{\Sigma}_{x}^{t,k}\big(\nabla\mu_{u}(\tilde{\mu}_{x}^{t,k})\big)^\top\\
\tilde{\Sigma}_{xu}^{t,k}&=\tilde{\Sigma}_x^{t,k}(\nabla\mu_u(\tilde{\mu}_x^{t,k}))^\top.
\end{split}\label{Taylor_app}
\end{equation} 

To propagate the obstacle's state with the new distribution information about $u_o(t+k)$, we perform the following update starting from the current state $x_o(t)$: Set $\tilde{\mu}_x^{t,0} = x_o(t)$ and $\tilde{\Sigma}_x^{t,0} = \mathbf{0}$, and successively linearize $\phi$ around $(\tilde{\mu}_x^{t,k},\tilde{\mu}_u^{t,k})$:
\begin{equation}
\begin{split}
\tilde{\mu}_x^{k+1}&=\phi(\tilde{\mu}_x^{t,k},\tilde{\mu}_u^{t,k}),\\
\tilde{\Sigma}_x^{k+1}&= \nabla_x \phi(\tilde{\mu}_x^{t,k},\tilde{\mu}_u^{t,k})\tilde{\Sigma}_x^{t,k}\nabla_x \phi(\tilde{\mu}_x^{t,k},\tilde{\mu}_u^{t,k})^\top +\nabla_u \phi(\tilde{\mu}_x^{t,k},\tilde{\mu}_u^{t,k})\tilde{\Sigma}_u^{t,k}\nabla_u \phi(\tilde{\mu}_x^{t,k},\tilde{\mu}_u^{t,k})^\top\\
&+2\nabla_x \phi(\tilde{\mu}_x^{t,k},\tilde{\mu}_u^{t,k})\tilde{\Sigma}_{x u}^{t,k}\nabla_u \phi(\tilde{\mu}_x^{t,k},\tilde{\mu}_u^{t,k})^\top.
\end{split}\label{x_mu_sigma}
\end{equation}
The corresponding mean and covariance of the obstacle's output $y_o(t+k)$ are computed by
\begin{equation}
\tilde{\mu}_{y}^{t,k}=C_o \tilde{\mu}_{x}^{t,k},\quad 
\tilde{\Sigma}_{y}^{t,k}=C_o \tilde{\Sigma}_{x}^{t,k} (C_o)^\top\label{y_mu_sigma}.
\end{equation}

\begin{figure*}[tb]
\centering
     \begin{subfigure}[b]{0.33\linewidth}
         \centering
         \includegraphics[width=\linewidth]{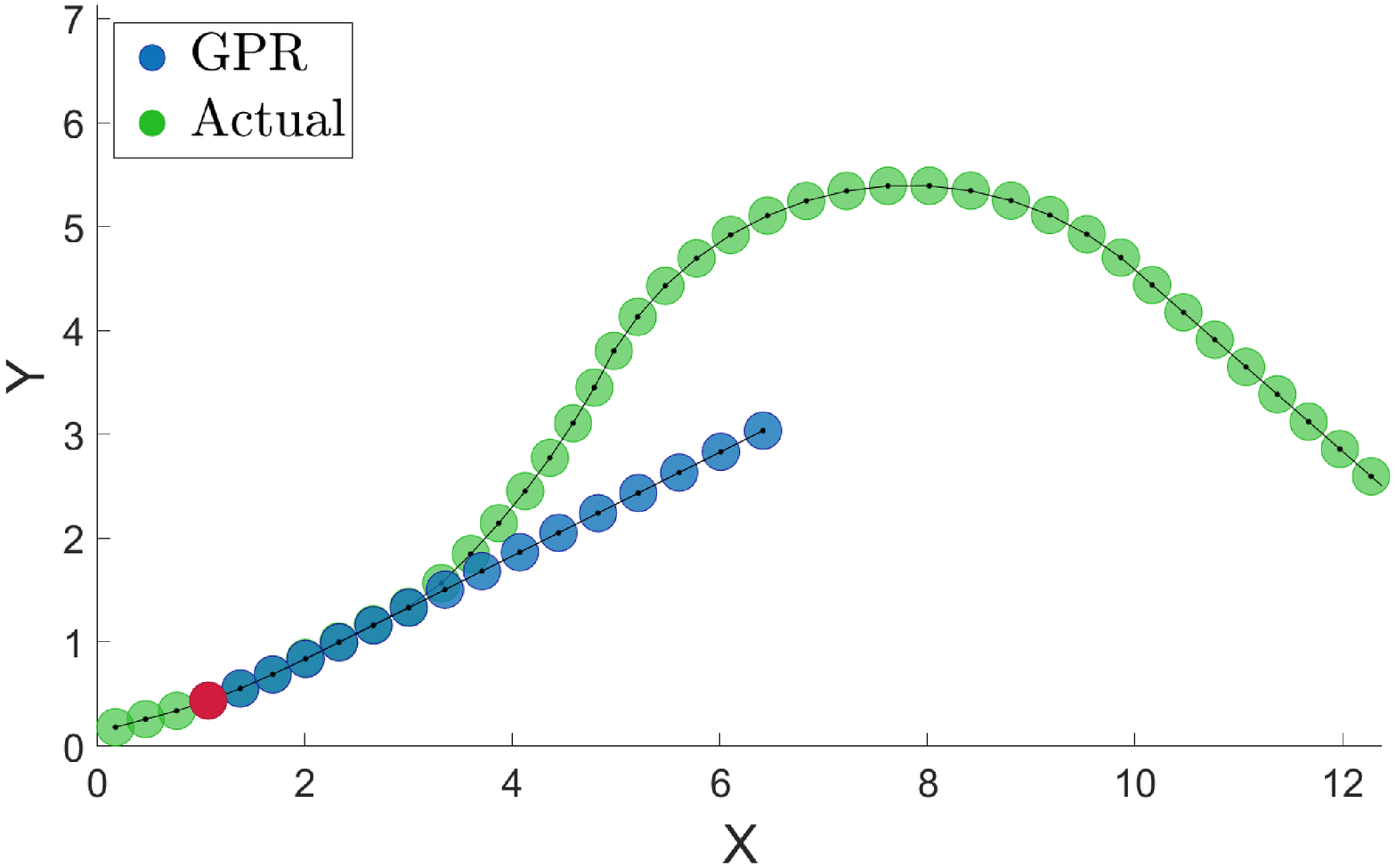}
         \caption{$t=3$}
         \label{fig:GP_3}
     \end{subfigure}%
     \begin{subfigure}[b]{0.33\linewidth}
         \centering
         \includegraphics[width=\linewidth]{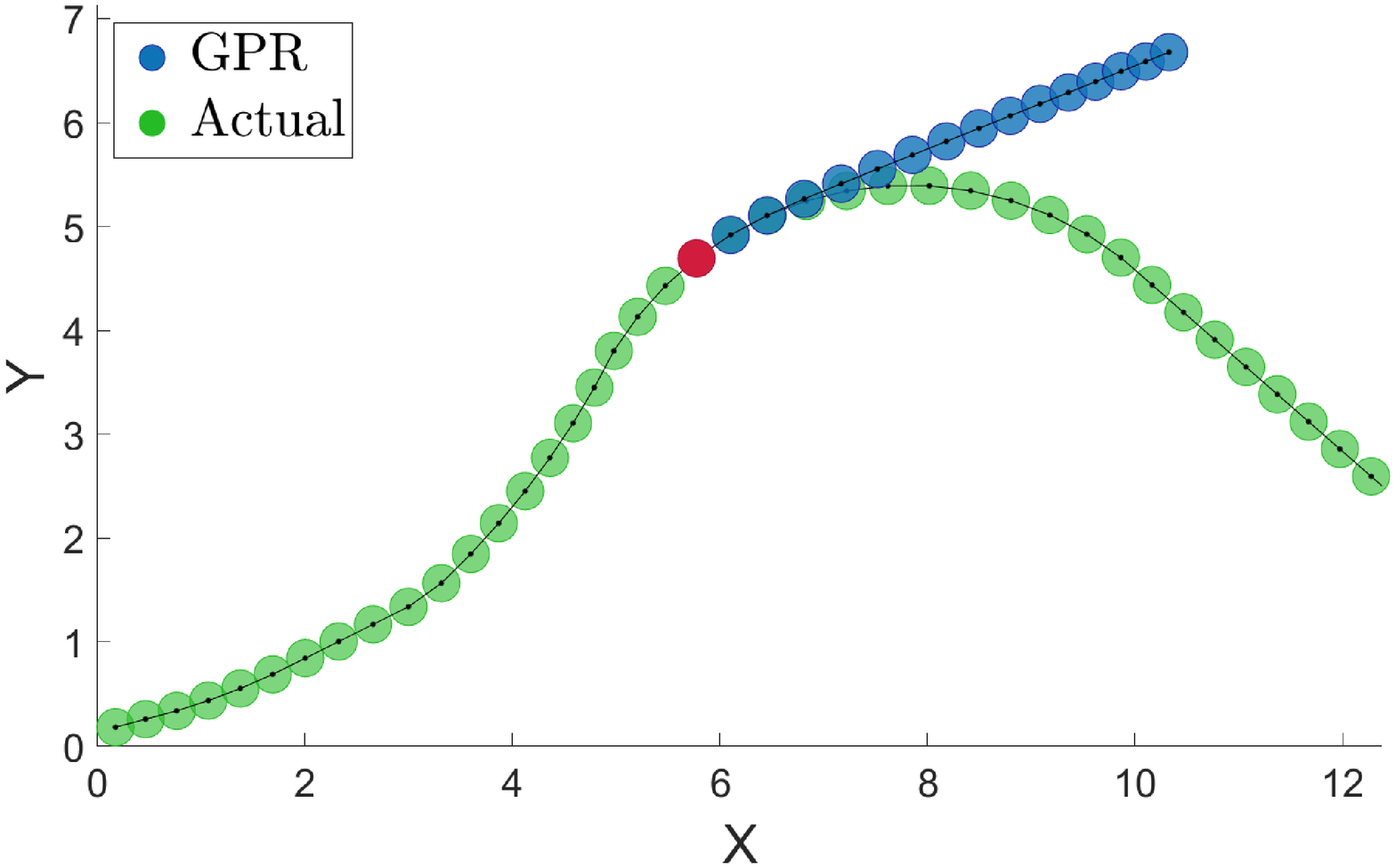}%
         \caption{$t=20$}
         \label{fig:GP_20}
     \end{subfigure}%
     \begin{subfigure}[b]{0.33\linewidth}
         \centering
         \includegraphics[width=\linewidth]{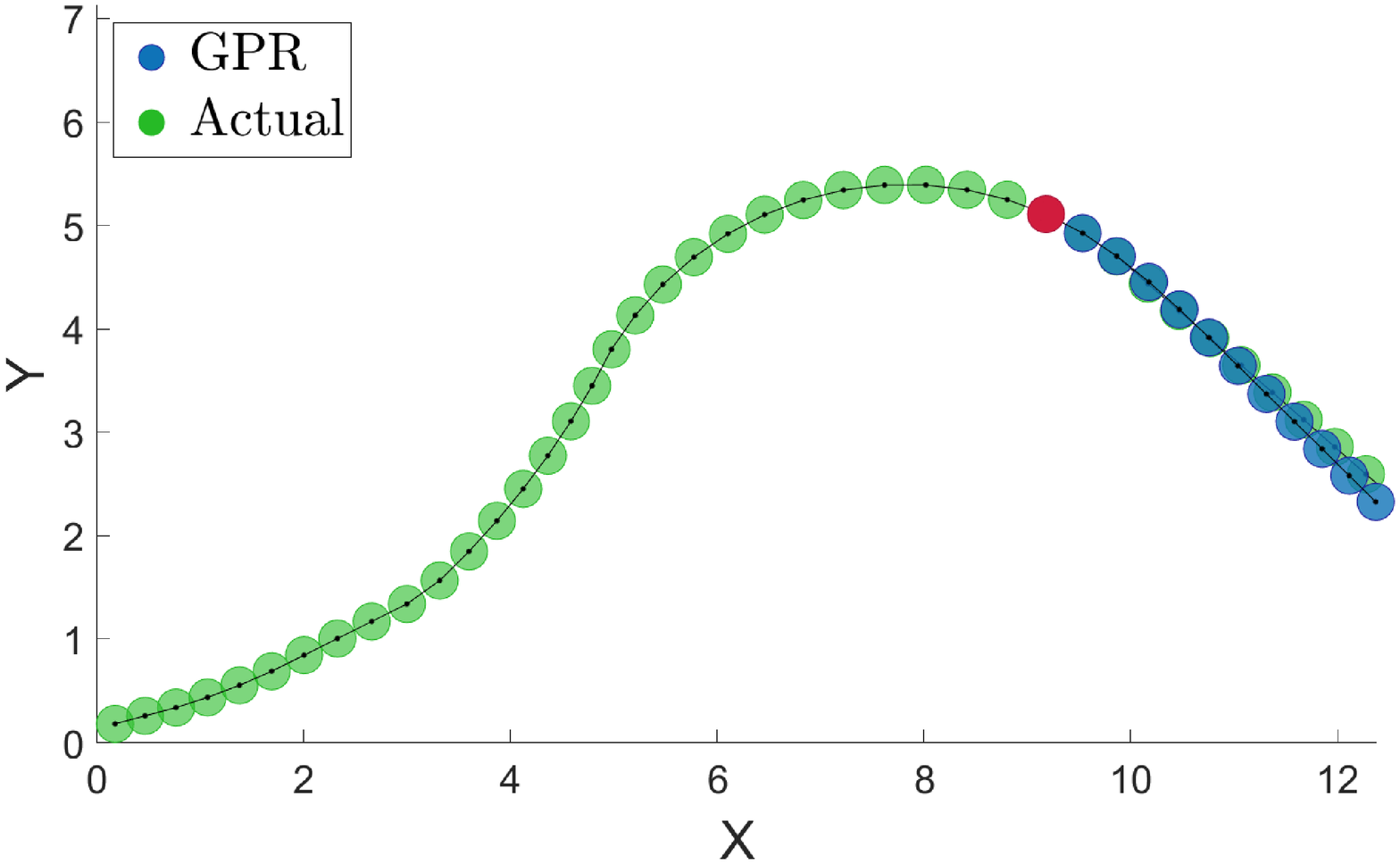}
         \caption{$t=29$}
         \label{fig:GP_29}
     \end{subfigure}%
        \caption{Trajectories of an obstacle predicted using GPR. 
        The mean of each trajectory is represented by a point, while the covariance is represented by an ellipsoid.}
        \label{fig:GP_example}
\end{figure*}

Fig.~\ref{fig:GP_example} shows an example of inferring an obstacle's motion via GPR for $K=15$ time steps, where the obstacle moves according to simple car dynamics. As shown in Fig.~\ref{fig:GP_example}~(a), the predicted trajectory in the early stages does not follow the actual trajectory, as we do not have much information from the previous observations. However, even when more data are collected, it is impossible to predict the curvature of the trajectory due to a sudden change in the heading angle of the obstacle (Fig.~\ref{fig:GP_example}~(b)).
 As time goes on and there are no sudden changes in the obstacle's behavior, the learned trajectory becomes closer to the actual trajectory (Fig.~\ref{fig:GP_example}~(c)). As illustrated in this example, the prediction results of GPR are not always reliable. To guarantee safety even in such cases, we propose a distributionally robust approach in the following section.

\section{Distributionally Robust Risk Map with Wasserstein Distance}\label{sec:risk_map}

To perform safe motion planning and control, the robot may want to estimate the risk of collision at any location in the configuration space with respect to the $L$ obstacles. 
However, it is challenging to measure the risk of collision in a reliable way since the results of GPR may be inaccurate, as demonstrated in the previous section.
To resolve this issue, we propose the \emph{distributionally robust risk map}, which is a spatially varying function of the robot's current position. It estimates the conditional value-at-risk (CVaR) of collision in a distributionally robust manner using the possibly erroneous results of GPR.

\subsection{Measuring the Risk of Collision Using CVaR}

To begin, we define the \emph{loss of safety} at each prediction time $t+k$,  evaluated at $t$, with respect to obstacle $\ell$ as
\begin{equation}
J_{t,k} (y_r, y^{\ell}_o)=-\|y_r(t+k)-y^{\ell}_o(t+k)\|_2^2.
\end{equation}
It follows from \eqref{safereg} that $J_{t,k} (y_r, y^{\ell}_o)+r_{\ell}^2$ is non-positive if and only if the robot navigates in the safe region $\mathcal{Y}^\ell (t+k)$. 
However, due to the uncertainty in the predicted $y^{\ell}_o(t+k)$, 
it may be too conservative to impose the deterministic constraint 
$J_{t,k} (y_r, y^{\ell}_o) +r_{\ell}^2 \leq 0$.

Instead, we consider the CVaR of the loss of safety, defined by
\begin{equation} \nonumber
\begin{split}
&\mathrm{CVaR}^{\mathrm{P}_{t,k}^{\ell}}_\alpha\big[ J_{t,k} (y_r, y^{\ell}_o) \big] := \min_{z \in \mathbb{R}} \mathbb{E}^{\mathrm{P}_{t,k}^{\ell}} 
\bigg [
z + \frac{( J_{t,k} (y_r, y^{\ell}_o) - z)^+}{1-\alpha}
\bigg ],
\end{split}
\end{equation}
where  $\mathrm{P}_{t,k}^{\ell}$ is the probability distribution of $y_o^\ell (t+k)$, estimated by GPR~\eqref{y_mu_sigma} at time $t$, and $(z)^+ := \max\{z, 0\}$.
The CVaR of $J_{t,k} (y_r, y^{\ell}_o)$ measures the conditional expectation of the loss within the $(1-\alpha)$ worst-case quantile as illustrated in Fig.~\ref{fig:CVaR}. 
Thus, if $\mathrm{CVaR}^{\mathrm{P}_{t,k}^{\ell}}_\alpha\big[ J_{t,k} (y_r, y^{\ell}_o) \big]+r_\ell^2 \leq 0$, then the robot is located in the safe region with a probability of no less than $\alpha$. 
As it takes into account the tail distribution through conditional expectation,
the CVaR constraint is capable of distinguishing rare events compared to chance constraints~\cite{Rockafellar2002a}. 
Furthermore, CVaR is \emph{coherent} in the sense of Artzner {\it et al.}~\cite{Artzner1999} and  is advocated as a rational risk measure in robotics applications, unlike the value-at-risk (VaR), or, equivalently, chance constraints~\cite{Majumdar2017isrr}.
Thus, CVaR has recently received a considerable attention in robotic decision-making problems for safety~\cite{sharma2020risk, singh2018risk, jin2019risk, hakobyan2020wasserstein, ahmadi2020}.

In practice, it is unlikely that we can accurately compute the CVaR of the loss of safety since $\mathrm{P}_{t,k}^{\ell}$ obtained by GPR is imperfect. 
To handle such distribution errors, 
we propose using the following distributionally robust version of CVaR:
\begin{equation}\label{DR_risk}
\dcvar_{\alpha,\theta}\big[J_{t,k} (y_r,y^{\ell}_o)\big]:=
\sup_{\mathrm{Q}_{t,k}^{\ell}\in\mathbb{D}_{t,k}^{\ell}}\mathrm{CVaR}^{\mathrm{Q}_{t,k}^{\ell}}_{\alpha}\big[J_{t,k} (y_r,y^{\ell}_o)\big],
\end{equation}
which measures the risk of unsafety for the worst-case distribution in a an ambiguity set $\mathbb{D}_{t,k}^\ell$. We consider the \emph{Wasserstein ambiguity set}, constructed as a ball with radius $\theta>0$ around the nominal distribution  $\mathrm{P}_{t,k}^{\ell}$, obtained by GPR, i.e.,
\begin{align}
\mathbb{D}_{t,k}^{\ell}:=\{\mathrm{Q}\in\mathcal{P}(\mathbb{R}^{n_y})\mid W_2(\mathrm{Q},\mathrm{P}_{t,k}^{\ell})\leq \theta\},\label{WBall}
\end{align}
where  $W_2(\mathrm{Q},\mathrm{P}_{t,k}^{\ell})$ is the 2-Wasserstein distance between $\mathrm{Q}$ and $\mathrm{P}_{t,k}^{\ell}$. 
The $p$-Wasserstein metric $W_p (\mathrm{Q},\mathrm{P})$ between two distributions $\mathrm{Q}$ and $\mathrm{P}$ supported on $\Xi\subseteq \mathbb{R}^m$ is defined as
\begin{equation} \nonumber
\begin{split}
W_p (\mathrm{Q}, \mathrm{P}) := \bigg[ \min_{\kappa \in \mathcal{P}(\Xi^2)} \Big \{
& \int_{\Xi^2} \| y - y' \|^{p} \; \mathrm{d} \kappa (y, y') \mid \Pi^1 \kappa = \mathrm{Q}, \Pi^2 \kappa = \mathrm{P}
\Big \}\bigg]^{1/p},
\end{split}
\end{equation}
where $\kappa$ is the transportation plan, the $i$th marginal of which is denoted by $\Pi^i \kappa$. It represents the minimum energy cost for transporting the mass from $\mathrm{Q}$ to $\mathrm{P}$ with the cost of moving a unit mass from position $y$ to position $y'$ prescribed by $\|y-y'\|^p$, where $\|\cdot\|$ is a norm on $\mathbb{R}^{m}$.

It is well known that for the standard Euclidean norm $\| \cdot \|_2$ the 2-Wasserstein distance between two normal distributions $\mathrm{Q}=\mathcal{N}(\mu_1,\Sigma_1)$ and $\mathrm{P}=\mathcal{N}(\mu_2,\Sigma_2)$ has a closed-form expression~\cite{gelbrich1990formula}:
\begin{equation}\nonumber
W_{2}(\mathrm{Q},\mathrm{P}) = \sqrt{\|\mu_1 - \mu_2\|_2^2+B^2(\Sigma_1,\Sigma_2)},
\end{equation}
where
\[
B^2(\Sigma_1,\Sigma_2):=\mathrm{Tr}\Big[\Sigma_1 +\Sigma_2 -2\big(\Sigma_1^{1/2}\Sigma_2 \Sigma_1^{1/2}\big)^{1/2}\Big].
\]

\begin{figure}[t!]
\centering
\includegraphics[width=2.75in]{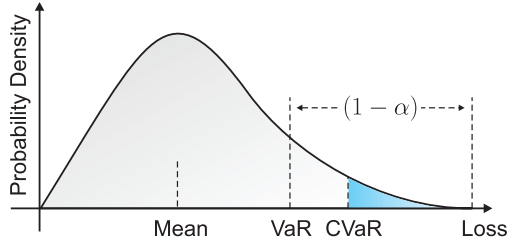}
\caption{Conditional value-at-risk of a random loss}
\label{fig:CVaR}
\end{figure}

The Wasserstein metric is also known as the \emph{earth mover's distance}, as it can be interpreted as the minimum cost of turning one pile of earth into another, where each distribution is viewed as a unit amount of earth. The Wasserstein ambiguity sets have several advantages over other types of ambiguity sets.
First, the Wasserstein DRO problem is capable of anticipating realizations of uncertainty that differ from the predicted ones, unlike DRO approaches that use phi-divergence~\cite{gao2016distributionally}. Second, Wasserstein ambiguity sets provide a powerful finite sample guarantee  for empirical nominal distributions and this feature is useful in sequential decision-making problems~\cite{esfahani2018data, yang2020wasserstein, Kim2021}.
Third, Wasserstein DRO is strongly related to the regularization techniques in machine learning and can be applied to alleviate overfitting~\cite{kuhn2019wasserstein}.

Concerning all the obstacles, we define the \emph{distributionally robust risk map} (DR-risk map)
$\mathcal{R}_{t,k}: \mathbb{R}^{n_y} \to \mathbb{R}$  for prediction time $t+k$, evaluated at $t$,  as
\begin{equation} 
\mathcal{R}_{t,k}(y_r):=\max_{\ell=1,\dots,L}\mathcal{R}^{\ell}_{t,k}(y_r,\mathcal{Y}^{\ell}),\label{mod_DRCVaR}
\end{equation}
where
\begin{equation}
\mathcal{R}_{t,k}^{\ell}(y_r,\mathcal{Y}^{\ell}):=\Big(\dcvar_{\alpha,\theta}\big[J_{t,k} (y_r,y^{\ell}_o)\big] +r_{\ell}^2\Big)^+. \label{pos_risk}
\end{equation}
The DR-risk map returns the maximum risk for all obstacles.  Its value is zero if there is no risk; otherwise, its value is positive.
In our safe motion planning and control methods, the following constraint is used to limit the risk of collision:
\[
\mathcal{R}_{t,k}(y_r) \leq \delta,
\]
where $\delta \geq 0$ is a risk tolerance parameter.

\subsection{Semidefinite Programming Formulation}

Unfortunately, it is nontrivial to directly compute the DR-risk map $\mathcal{R}_{t,k}(y_r)$ or its proxy\\
$\dcvar_{\alpha,\theta} [J_{t,k} (y_r,y^{\ell}_o)]$
as this involves an infinite-dimensional optimization problem over the set of probability distributions. 
We reformulate it as a finite-dimensional problem by exploiting some structural properties of CVaR and Wasserstein distance. 
The following theorem presents the result of reformulation as a semidefinite program (SDP), where the dependence on $t,k$ and $\ell$ is encoded solely in $y_r(t+k), \tilde{\mu}_y^{t,k,\ell}$ and $\tilde{\Sigma}_y^{t,k,\ell}$. Later, this feature will allow us to approximate the risk map by a single NN, independent of $t,k$ and $\ell$.

\begin{theorem}\label{thm:sdp}
Let  $\mathrm{P}_{t,k}^{\ell}$ be the  distribution of $y_o^{\ell}$ with mean $\tilde{\mu}_y^{t,k,\ell}$ and covariance $\tilde{\Sigma}_y^{t,k,\ell}$, estimated by GPR. Then, the DR-CVaR~\eqref{DR_risk} has the following upper bound:
\begin{equation}\label{DR-CVaR}
\begin{split}
\min  \;& z+\frac{\tau+\varepsilon+\mathrm{Tr}[Z]+\lambda\big(\theta^2-\|\tilde{\mu}_y^{t,k,\ell}\|_2^2-\mathrm{Tr}[\tilde{\Sigma}_y^{t,k,\ell}]\big)}{1-\alpha}\\
\textnormal{s.t.}\;& \begin{bmatrix}\lambda I-\Gamma & \gamma+\lambda\tilde{\mu}_y^{t,k,\ell}\\ \big( \gamma+\lambda \tilde{\mu}_y^{t,k,\ell}\big)^\top & \varepsilon\end{bmatrix}\succeq 0\\
& \begin{bmatrix}\lambda I-\Gamma & \lambda\big(\tilde{\Sigma}_y^{t,k,\ell}\big)^{1/2} \\ \lambda\big(\tilde{\Sigma}_y^{t,k,\ell}\big)^{1/2} & Z\end{bmatrix}\succeq 0\\
& \begin{bmatrix}\Gamma+I & \gamma-y_r(t+k)\\ \big(\gamma -y_r(t+k) \big)^\top & \tau+z+\|y_r(t+k)\|_2^2\end{bmatrix}\succeq 0\\
& \begin{bmatrix}\Gamma & \gamma\\ \gamma^\top & \tau\end{bmatrix}\succeq 0\\
&\lambda\in\mathbb{R}_+, \; z\in\mathbb{R}, \; \tau\in\mathbb{R},\; \gamma\in\mathbb{R}^{n_y}\\
&\Gamma\in\mathbb{S}^{n_y}, \; \varepsilon\in\mathbb{R}_+, \; Z\in\mathbb{S}_+^{n_y}.
\end{split}
\end{equation}
\end{theorem}

Its proof is contained in Appendix~\ref{appendix2}.
The SDP problem \eqref{DR-CVaR} can be solved using well-known algorithms, such as interior-point methods~\cite{andersen2003implementing, toh1999sdpt3, sturm1999using},  splitting methods~\cite{o2016conic}, augmented Lagrangian methods~\cite{kovcvara2003pennon}, etc. Its dual problem is more of an interest, as it involves fewer generalized equalities.

\begin{cor}\label{cor:sdp}
The dual problem of \eqref{DR-CVaR} can be expressed as the following SDP:
\begin{equation}
\begin{split}
\max \; & 2W_{12}^\top y_r(t+k) -\mathrm{Tr}[W_{11}]- \|y_r(t+k)\|_2^2\\
\textnormal{s.t. } \; & \frac{1}{1-\alpha}\big(\theta^2 - \|\tilde{\mu}_y^{t,k,\ell}\|_2^2-\mathrm{Tr}[\tilde{\Sigma}_y^{t,k,\ell}]\big)-2X_{12}^\top \tilde{\mu}_y^{t,k,\ell} \\
&-\mathrm{Tr}[X_{11}+Y_{11}+2Y_{12}^\top \big(\tilde{\Sigma}_y^{t,k,\ell} \big)^{1/2}]\geq 0\\
& X_{11}+Y_{11}=W_{11}+V_{11}\\
& X_{12}+W_{12}+V_{12} = 0 \\
& W_{22} = 1, \; V_{22} = \frac{1}{1-\alpha} - 1\\
& X_{22} \leq \frac{1}{1-\alpha}, \; Y_{22} \preceq \frac{1}{1-\alpha}I\\
& Y\in\mathbb{S}^{2n_y}_+, \; X,W,V\in\mathbb{S}^{n_y+1}_+.
\end{split}\label{DR-CVaR_dual}
\end{equation}
\end{cor}

Its proof can be found in Appendix~\ref{appendix2}.
The dual problem is also a tractable SDP problem, which can be solved using the same algorithms as for the primal. However,  the dual problem~\eqref{DR-CVaR_dual} has only less linear matrix inequality constraint in addition to a number of linear equality and inequality constraints, which are easier to handle for most of the off-the-shelf solvers than the positive semidefinite constraints in the primal problem~\eqref{DR-CVaR}. Even though strong duality is not guaranteed to hold for all $\tilde{\mu}_y^{t,k,\ell}$ and $\tilde{\Sigma}_y^{t,k,\ell}$, the dual problem is still useful since in some cases the SDP solver might fail to solve~\eqref{DR-CVaR} due to numerical issues. We can use the solution to the dual problem if there is no primal solution returned by the solver.

\begin{figure*}[tb]
\centering
\includegraphics[width=6in]{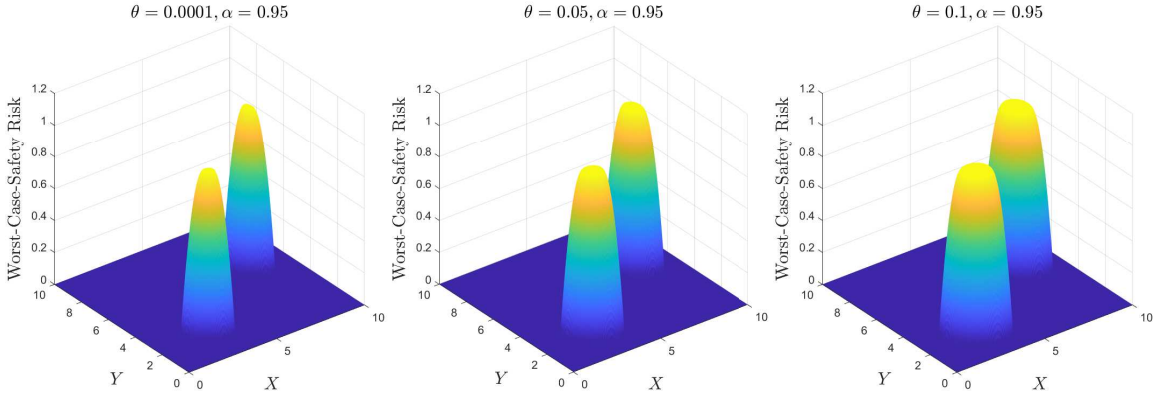}
\caption{Risk maps for two obstacles with means $\tilde{\mu}_y^{t,k,1}=(3, 2.5)$, $\tilde{\mu}_y^{t,k,2}=(8, 6)$ and covariances $\tilde{\Sigma}_y^{t,k,1}=\mathrm{diag}[0.003, 0.002]$, $\tilde{\Sigma}_y^{t,k,2}=\mathrm{diag}[0.001, 0.004]$ for $\theta=\{0.0001,0.05,0.1\}$ and  $\alpha=0.95$.}
\label{fig:riskmap1}
\end{figure*}

\begin{figure*}[tb]
\centering
\includegraphics[width=6in]{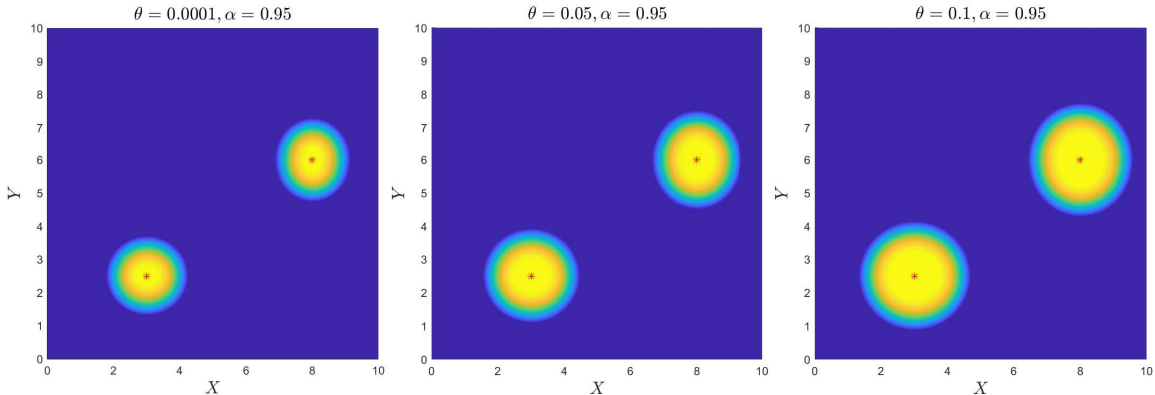}
\caption{Projection of the risk maps onto the robot's configuration space.}
\label{fig:riskmap2}
\end{figure*}

\subsection{Example of DR-Risk Maps}\label{sec:example}

By discretizing the robot's configuration space  and solving either \eqref{DR-CVaR} or \eqref{DR-CVaR_dual} for all discretized points, we can construct the desired DR-risk map~\eqref{mod_DRCVaR}. 
Fig.~\ref{fig:riskmap1} shows examples of such risk maps, which are obtained by solving the primal problem for a risk confidence level $\alpha=0.95$ with two obstacles ($L=2$) at stage $t+k$. In the shown risk maps,  the estimated means and covariances for two obstacles' CoMs are set to $\tilde{\mu}_y^{t,k,1}=[3, 2.5], \tilde{\mu}_y^{t,k,2}=[8, 6]$ and $\tilde{\Sigma}_y^{t,k,1}=\mathrm{diag}[0.003, 0.002], \tilde{\Sigma}_y^{t,k,2}=\mathrm{diag}[0.001, 0.004]$, respectively. Each peak of the risk map is  located at the mean of each obstacle's CoM with a value of $r_\ell^2 = 1$. 
The risk diminishes as the robot moves  away from the obstacle.
Fig.~\ref{fig:riskmap1} demonstrates that  the non-zero area of the risk map expands as the radius $\theta$ increases. 
 Also, the peak area for a bigger radius becomes flatter, meaning that more regions are considered  ``risky". 
 Therefore, the robot's decision using this map will be more robust against errors in the estimated distribution as the Wasserstein ambiguity set gets bigger.
 
Fig.~\ref{fig:riskmap2} shows the projection of the risk map onto the robot's configuration space. It shows that a bigger $\theta$ generates a more conservative risk map.  The risky area enlarges with the size of our ambiguity set. 

For an efficient construction of the risk map, we propose an NN approximation in Section~\ref{sec:app}.
The NN approach avoids any discretization of the robot's configuration space 
or training of multiple networks for different $t, k$ and $\ell$ 
because such dependence is encoded in $y_r (t+k), \tilde{\mu}_y^{t, k, \ell}$ and $\tilde{\Sigma}_y^{t,k,\ell}$ as previously mentioned. 
In the following two sections, we present applications of the DR-risk map to safe motion planning and control in learning-enabled environments.

\section{Application to Learning-Based Distributionally Robust Motion Planning}\label{sec:LBDRMP}

As the first application of the DR-risk maps, we propose a learning-based motion planning algorithm based on RRT*~\cite{karaman2011sampling}. 
 Unlike previous RRT algorithms, our  algorithm
takes into account possible errors in the learned distribution of the obstacles' behaviors.

\subsection{Main Algorithm}

The motion planning algorithm presented in this section is an online sampling-based algorithm for computing a path from the robot's starting point to the goal point in near real-time, taking into account moving obstacles. 
The overall algorithm, similar to the original RRT* algorithm, consists of the nearest neighbor search, steering towards the sampled node, safety check and rewiring. Inspired by~\cite{karaman2011anytime}, the path is generated only for a given time, after which the robot executes the committed trajectory and restarts the planning process from a new initial state, removing unreachable nodes from the tree. The key extension to the original algorithm is the use of the DR-risk maps for safety checks. In addition, the algorithm leverages GPR to infer the future trajectories of the obstacles based on either the system dynamics~\eqref{obs_model} or its approximation~\eqref{app_obs_model}. The risk map in~\eqref{mod_DRCVaR} is employed to guarantee the safety of the derived paths in two stages. First, each node computed in the growing stage of the tree is classified as either safe or unsafe based on the risk value to later include it in or exclude it from the safe subtree. Second, the cost function of planning includes the risk value to escape possibly unsafe nodes. The robot's dynamics~\eqref{sysmod} as well as the constraints on the state and input are incorporated into tree expansion, as the steering towards a sampled node is performed according to the given dynamics by applying control actions that satisfy the constraints. 
Moreover, when changing the parent from one node to another, the feasibility of the trajectories and control actions are checked once again to meet the given requirements.

Our learning-based distributionally robust RRT* (DR-RRT*) algorithm is presented in Algorithm~\ref{DR_Risk_RRT}, given goal state $q_{\mathrm{goal}}$, maximum depth $K$, risk weight constant $w$, other hyper-parameters $\theta, \alpha$ and $r_{\ell}$ for computing risk, as well as the radius $r_{\mathrm{RRT}}$ for neighborhood construction, computed as in~\cite[Theorem 38]{karaman2011sampling}.

At the beginning of the algorithm, $\mathcal{T}$ is set as an empty tree to be expanded later. Initially, the GP dataset $D^{\ell}$ is also an empty set. 
In each iteration, a new safe subtree $\mathcal{T}_{\mathrm{safe}}$ is defined (Line~\ref{alg1:1}). Then, the robot's state   $x_r(t)$ as well as the obstacle's state and action $x_o^{\ell}(t)$ and $u_o^{\ell}(t)$ are observed at current stage $t$, as performed in Line~\ref{alg1:2}. Thereafter, the tree is constructed with $x_r(t)$ as the root (Line~\ref{alg1:3}). Since there might be some nodes that are unreachable from the current state, we remove the corresponding edges and vertices in Line~\ref{alg1:4}. 
These nodes are all nodes that do not root from the current state $x_r(t)$. 
In Line~\ref{alg1:5}, the pruned tree is updated with a new depth value starting from the root, the depth of which is set to $k=0$.

\begin{algorithm}[t]
\SetKw{Input}{Input:}
\Input $q_{\mathrm{goal}},K, \theta,\alpha,r_{\ell},w,r_{\mathrm{RRT}}$\;
$\mathcal{T}=\emptyset,\mathcal{D}^{\ell}\leftarrow \emptyset$\;
\While{$\|\mathrm{Root}(\mathcal{T})-q_{\mathrm{goal}}\|_2 > \epsilon$}{
$t\leftarrow \mathrm{clock}()$\;
$\mathcal{T}_{\mathrm{safe}}\leftarrow \emptyset$\; \label{alg1:1}
Observe $x_r(t)$ and $x^{\ell}_o(t),u^{\ell}_o(t)$ for all $\ell$\;\label{alg1:2}
$\mathrm{Root} (\mathcal{T}) \leftarrow x_r(t)$\;\label{alg1:3}
Remove unreachable nodes from $\mathcal{T}$\;\label{alg1:4}
Reset node depth\;\label{alg1:5}
\For{$\ell=1$ to $L$}{ \label{alg1:6}
$\mathcal{D}^{\ell}_j \leftarrow \mathcal{D}^{\ell}_j \cup \big\{(x_o^{\ell}(t),u_{o,j}^{\ell}(t))\big\},\;j=1,\dots, n_u^{\ell}$\;
GP approximation of $\psi^{\ell}(\mathbf{x})$ via \eqref{mean}--\eqref{cov}\;
$\tilde{\mu}_x^{t,0,\ell}\leftarrow x_{o}^{\ell}(t),\tilde{\Sigma}_x^{t,0,\ell}\leftarrow \mathbf{0}$\;
\For{$k=0$ to $K-1$}{
Compute $\tilde{\mu}_u^{t,k,\ell}$, $\tilde{\Sigma}_u^{t,k,\ell}$ and $\tilde{\Sigma}_{xu}^{t,k,\ell}$ from \eqref{Taylor_app}\;\label{alg1:8}
Update $\tilde{\mu}_y^{t,k+1,\ell}$ and $\tilde{\Sigma}_y^{t,k+1,\ell}$ by \eqref{x_mu_sigma}--\eqref{y_mu_sigma}\;\label{alg1:9}
}}\label{alg1:7}
\For{$\forall q\in \mathcal{T}$ with $\mathrm{Depth}(q)\leq K$}{\label{alg1:10}
$k\leftarrow \mathrm{Depth}(q)$\;
Update  $\mathcal{R}_{t,k}(C_r q)$ by solving \eqref{DR-CVaR}\label{alg1:riskmap}\;
Update $c(q)$ by solving \eqref{cost}\;
\If {$\mathcal{R}_{t,k}(C_r q)\leq \delta$}{
Add $q$ to $\mathcal{T}_{\mathrm{safe}}$\;\label{alg1:11}
}
}
\While{$\mathrm{clock}()\leq \tau$}{\label{alg1:12}
Expand the tree using Algorithm 2}\label{alg1:13}
Plan path $(\mathrm{Root}(\mathcal{T}_{\mathrm{safe}}),q_1,\dots,\dots,q_K)$ in $\mathcal{T}_{\mathrm{safe}}$\;\label{alg1:14}
Drive $x_r(t)$ to $q_1$\;\label{alg1:15}
}
\caption{Learning-based DR-RRT*}\label{DR_Risk_RRT}
\end{algorithm}

Having new perceived information about the obstacles' motions, we perform GPR  in Line~\ref{alg1:6}--\ref{alg1:7}. Here, the GP dataset is updated with new observations, after which the GP approximation of   $\psi^{\ell}(\mathbf{x})$ is updated by learning mean and covariance functions $\mu_u^{\ell,j}(\mathbf{x})$ and $\Sigma_{u}^{\ell,j}(\mathbf{x})$ as in \eqref{mean} and \eqref{cov}. To predict the trajectory of each obstacle starting from $t+1$ to $t+K$, the mean and covariance at $t$ are initialized as the current observation and the zero covariance matrix, respectively. 
In Line~\ref{alg1:8}--\ref{alg1:9}, the mean and the covariance of the obstacle's action, state and output are computed by \eqref{Taylor_app}, \eqref{x_mu_sigma} and \eqref{y_mu_sigma}. Here, $K$ corresponds to the desirable time horizon or, equivalently, the maximum depth of the path.

Using the new prediction results, the safe tree updated in Line~\ref{alg1:10}--\ref{alg1:11} using the nodes of $\mathcal{T}$ satisfying the risk constraint $\mathcal{R}_{t,k} (C_r q) \leq \delta$ with depth less than threshold $K$. This is accomplished by calculating the DR-risk $\mathcal{R}_{t,k}(C_r q)$ for all nodes according to \eqref{mod_DRCVaR}, where the SDP problem \eqref{DR-CVaR} or its dual \eqref{DR-CVaR_dual} needs to be solved for each obstacle. Here, $k$ corresponds to the depth of the node, and therefore the predictions of step $k$ are used to compute the risk for a node of depth $k$. The new value of risk is used to update the costs for the corresponding nodes.

Next, in Line~\ref{alg1:12}--\ref{alg1:13}, we proceed to the expansion of the tree $\mathcal{T}$ for some fixed time $\tau$, where $\mathcal{T}_{\mathrm{safe}}$ is also updated with new nodes. The details of the tree expansion are given in Algorithm~\ref{Exp_Rew} and explained in Section~\ref{sec:tree}.

\begin{algorithm}[t]
\SetKw{Input}{Input:}
\SetKw{Output}{Output:}
\Input $\mathcal{T},\mathcal{T}_{\mathrm{safe}}, t$\;
$q_{\mathrm{rand}}\leftarrow \mathrm{Sample}()$\;\label{alg2:1}
$q_{\mathrm{nearest}}\leftarrow \mathrm{NearestNeighbor}(\mathcal{T}_{\mathrm{safe}},q_{\mathrm{rand}})$\;\label{alg2:2}
$k\leftarrow \mathrm{Depth}(q_\mathrm{nearest})+1$\;\label{alg2:3}
$(q_{\mathrm{new}},c(q_{\mathrm{new}}), \mathcal{R}_{t,k}(C_r q_{\mathrm{new}}))\leftarrow \mathrm{Steer}(q_{\mathrm{nearest}}, q_{\mathrm{rand}})$\;\label{alg2:4}
$\mathcal{N}_{\mathrm{near}} \leftarrow \mathrm{Near}(\mathcal{T}_{\mathrm{safe}}, q_{\mathrm{new}}, r_{\mathrm{RRT}})$\;\label{alg2:5}
$q_{\mathrm{min}}\leftarrow q_{\mathrm{nearest}}, c_{\mathrm{min}}\leftarrow c(q_{\mathrm{new}})$\;\label{alg2:6}
\For{${q}_{\mathrm{near}}\in\mathcal{N}_{\mathrm{near}}$}{
$k \leftarrow \mathrm{Depth}(q_\mathrm{near})+1$\;\label{alg2:7}
$c_{\mathrm{near}}\leftarrow c(q_{\mathrm{near}})+w \mathcal{R}_{t,k}(C_r q_{\mathrm{new}})+\mathcal{L}(q_{\mathrm{near}},q_{\mathrm{new}})$\;\label{alg2:8}
\If{$c_{\mathrm{near}}< c_{\mathrm{min}}$ and $\mathrm{Feas}(q_{\mathrm{near}},q_{\mathrm{new}})$}{
$q_{\mathrm{min}} \leftarrow q_{\mathrm{near}},\;c_{\mathrm{min}}\leftarrow c_{\mathrm{near}}$\;\label{alg2:9}
}
}
$c(q_{\mathrm{new}}) \leftarrow c_{\mathrm{min}}, \mathrm{Parent}(q_{\mathrm{new}}) \leftarrow q_{\mathrm{min}}$\;\label{alg2:10}
$k \leftarrow \mathrm{Depth}(q_\mathrm{new})$\;
Add $q_{\mathrm{new}}$ to $\mathcal{T}$\;
\If{$\mathcal{R}_{t,k}(C_r q_\mathrm{new})\leq \delta$}{\label{alg2:11}
Add $q_\mathrm{new}$ to $\mathcal{T}$\;\label{alg2:12}
}
\For{${q}_{\mathrm{near}}\in\mathcal{N}_{\mathrm{near}}$}{\label{alg2:13}
$k\leftarrow \mathrm{Depth}(q_\mathrm{new})+1$\;
$c_{\mathrm{min}}\leftarrow c(q_{\mathrm{new}})+w \mathcal{R}_{t,k}(C_r q_{\mathrm{near}})+\mathcal{L}(q_{\mathrm{new}},q_{\mathrm{near}})$\;
\If{$c_{\mathrm{min}}\leq c(q_{\mathrm{near}})$ and $\mathrm{Feas}(q_{\mathrm{new}},q_{\mathrm{near}})$}{
$c(q_{\mathrm{near}})=c_{\mathrm{near}}$\;
$\mathrm{Parent}(q_{\mathrm{near}})\leftarrow q_{\mathrm{new}}$\;
Update children nodes of $q_{\mathrm{near}}$ \;
\If{$\mathrm{Depth}(q_{\mathrm{near}})>K$}{\label{alg2:14}
Remove $q_{\mathrm{near}}$ and its children from $\mathcal{T}_{\mathrm{safe}}$\;\label{alg2:15}
}
}
}
\caption{Tree expansion and rewiring}\label{Exp_Rew}
\end{algorithm}

When the planning time is over, the best partial path is retrieved and passed to execution, being constructed from the root of the safe tree towards the goal  (Line~\ref{alg1:14}), where the current state corresponds to $q_0$. The robot follows the path for one step by driving it towards the next state $q_1$ in the planned path (Line~\ref{alg1:15}). The algorithm continues until the distance between the tree root (the current robot state) and the desired $q_\mathrm{goal}$ is no greater than tolerance $\epsilon$.

For real-time execution of the algorithm, it is necessary for the robot to operate while the planning is being performed. This can be achieved by executing Line~\ref{alg1:15} in parallel with the remaining parts of the algorithm. To ensure the termination of the algorithm, the tree will be grown until the planning time reaches $T_s$ seconds. 

\subsection{Tree Expansion and Rewiring}\label{sec:tree}

The tree expansion and rewiring algorithm is given in Algorithm~\ref{Exp_Rew}. Similar to the classical RRT*, the tree is expanded by randomly choosing a point in the configuration space (Line~\ref{alg2:1}). Then, in Line~\ref{alg2:2} the node to be extended is chosen as the minimizer of 
\begin{align*}
c(q,q_{\mathrm{rand}})&=c(q)+\mathcal{L}(q,q_{\mathrm{rand}}),
\end{align*}
where $\mathcal{L}(q,q_{\mathrm{rand}})$ is the length of the path from $q$ to $q_{\mathrm{rand}}$ and  $c(q)$ is the cost of node $q$, defined as
\begin{equation}
c(q)=c(\mathrm{Parent}(q))+w\mathcal{R}_{t,k}(C_r q)+\mathcal{L}(\mathrm{Parent}(q),q).\label{cost}
\end{equation}
The worst-case risk is taken into account in $c(q)$, where the SDP problem is solved for $(t+k)$th prediction performed at current stage $t$ with $k$ being the depth of node $q$.

In Line~\ref{alg2:3} the depth $k$ for the new node is set to the depth of the nearest node incremented by $1$ for computing risk in the next step. The new node $q_{\mathrm{new}}$ is obtained in Line~\ref{alg2:4} by steering the chosen best node towards $q_{\mathrm{rand}}$. Here, the control input is chosen as the one with the least cost $c(q_{\mathrm{new}})$.  The safety risk is given by \eqref{mod_DRCVaR} and computed by solving the SDP \eqref{DR-CVaR} or its dual \eqref{DR-CVaR_dual} for all $\ell=1,\dots,L$.

In Line~\ref{alg2:5}, the neighborhood of $q_{\mathrm{new}}$ is constructed from the nodes in safe subtree $\mathcal{T}_{\mathrm{safe}}$ with distance less than  $r_{\mathrm{RRT}}$ to $q_{\mathrm{new}}$. The best parent of $q_{\mathrm{new}}$ is chosen in Lines~\ref{alg2:6}--\ref{alg2:9}. 
The parent is initialized as $q_{\mathrm{nearest}}$. However, this is changed  if the cost to $q_{\mathrm{new}}$ via $q_{\mathrm{near}}$ is less than the cost via $q_{\mathrm{nearest}}$ and the new path is feasible. 
The node $q_{\mathrm{new}}$ with the updated parent is added to the tree in Line~\ref{alg2:10} only after selecting the parent. 
The subtree $\mathcal{T}_{\mathrm{safe}}$ is also updated  if the risk of the node $q_{\mathrm{new}}$ with depth $k$ is less than the threshold $\delta$ (Line~\ref{alg2:11}--\ref{alg2:12}).

Similar to the original RRT* algorithm,  the rewiring of the neighborhood nodes is performed in Line~\ref{alg2:13}--\ref{alg2:15} after the process of growing the tree is completed. For all $q_{\mathrm{near}}$ in $\mathcal{N}_{\mathrm{near}}$, the cost is calculated taking $q_{\mathrm{new}}$ as parent. 
If the new cost is less than the existing one and the path is feasible, the parent of $q_{\mathrm{near}}$ in both $\mathcal{T}$ and $\mathcal{T}_{\mathrm{safe}}$ is changed to $q_{\mathrm{new}}$.
The costs for $q_{\mathrm{near}}$ as well as its children nodes are updated to take into account the cost for $q_{\mathrm{new}}$. Unlike the original RRT* algorithm, in Line~\ref{alg2:14}--\ref{alg2:15} we also update the safe subtree, where the edge from $q_{\mathrm{new}}$ to $q_{\mathrm{near}}$ is added if the new depth is less than $K$. Otherwise, $q_{\mathrm{near}}$ is removed from the subtree to keep the safe subtree within the maximum depth $K$.

\subsection{Graphical Illustration}

\begin{figure*}[tb]
\centering
\includegraphics[width=\linewidth]{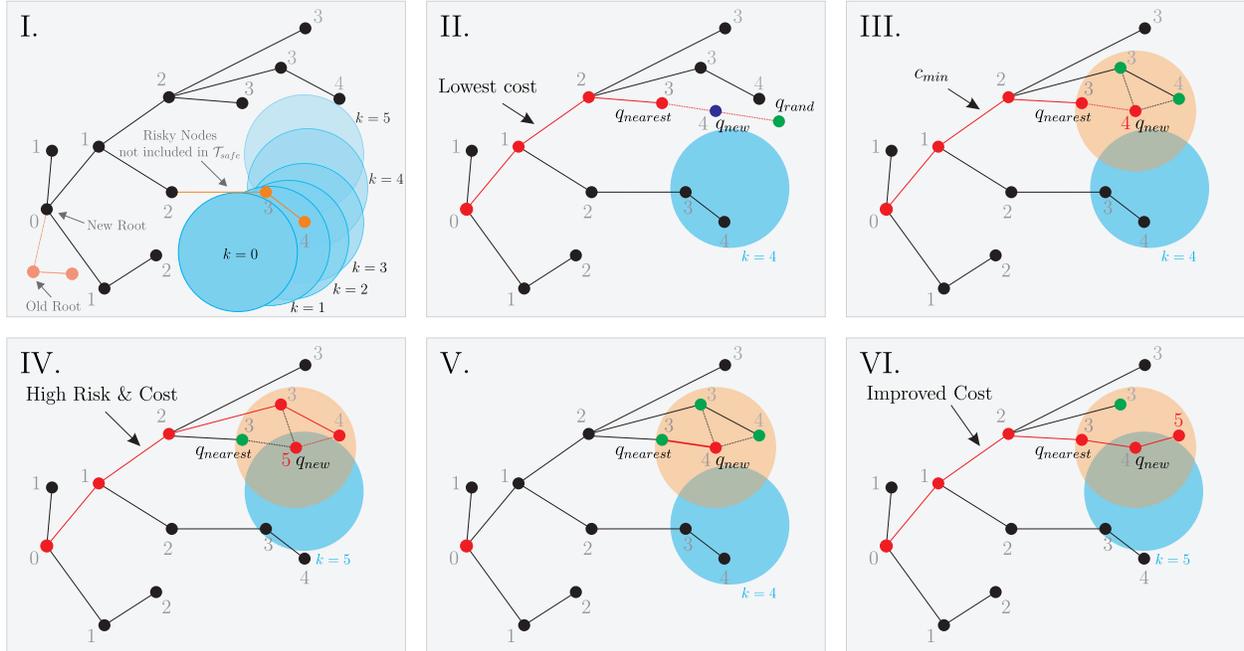}
\caption{Illustrative example of learning-based DR-RRT*. 
The blue ball represents an obstacle (at different time instances) centered at the predicted mean.} 
\label{fig:illustration}
\end{figure*}

A step-by-step example of our algorithm is illustrated in Fig.~\ref{fig:illustration}, where the blue ball represents an obstacle centered at the predicted mean at time steps $k=0,\dots,K$ with $K=5$. In Fig.~\ref{fig:illustration}-I, the robot is steered from the old root to  the new root. Thus, the part of the tree not growing from the new root is pruned.  
The vertices in orange with depth $3$ and $4$  have positive risks with respect to the predicted obstacle's location for $k=3$ and $k=4$, respectively. 
Hence, these nodes are not included in the safe subtree $\mathcal{T}_{\mathrm{safe}}$. 

In Fig.~\ref{fig:illustration}-II,  $q_{\mathrm{rand}}$ is sampled in the configuration space and the corresponding $q_{\mathrm{nearest}}$ is selected from $\mathcal{T}_{\mathrm{safe}}$ with the lowest cost. Finally, $q_{\mathrm{new}}$ is found by steering $q_\mathrm{nearest}$ towards $q_{\mathrm{rand}}$. 
In Fig~\ref{fig:illustration}-III, a ball of neighbors for $q_{\mathrm{new}}$ is created (in orange). This ball includes nodes in green as well as $q_{\mathrm{nearest}}$. The current lowest cost is set to the cost from the root of the tree to $q_{\mathrm{new}}$ via $q_{\mathrm{nearest}}$.

 In Fig~\ref{fig:illustration}-IV, the costs to $q_{\mathrm{new}}$ via other neighboring nodes are computed. 
 It is observed that the length to $q_{\mathrm{new}}$ and the risk are bigger via other neighbors than via $q_{\mathrm{nearest}}$. This is because the depth of $q_{\mathrm{new}}$ changes to $5$ and  the risk is computed for obstacles at $k=5$, whereas in the case of $q_{\mathrm{nearest}}$ being the parent, the depth of $q_{\mathrm{new}}$ is $4$ and the obstacle is farther from the node.
  Therefore, in Fig~\ref{fig:illustration}-V, the parent of $q_{\mathrm{new}}$ is chosen as $q_{\mathrm{nearest}}$. Also, $q_{\mathrm{new}}$ is added to the safe subtree since the risk is non-positive. Fig.~\ref{fig:illustration}-VI illustrates the rewiring process, where the cost for the neighbor node improves when its parent is changed to $q_{\mathrm{new}}$.

Our motion planning method is a learning-based algorithm based on  CC-RRT*~\cite{luders2013robust}, another real-time algorithm for probabilistically feasible motion planning built upon the chance constrained RRT (CC-RRT) algorithm~\cite{luders2010chance} and the original RRT*~\cite{karaman2011sampling}. 
Unlike CC-RRT*, our algorithm first learns the distribution of the obstacles' future trajectories from new observations and replaces the probability of collision by the distributionally robust risk map defined in~\eqref{mod_DRCVaR}. Then, instead of chance constraints,
the DR-risk map is used as a constraint to ensure safety as well as to penalize possibly risky trajectories in the cost function. It is well known that CVaR constraints induce more conservative behaviors compared to chance constraints. Moreover, our DR-risk map yields to take into account possible errors in the learned distribution of the obstacles' behaviors that in practice cannot be captured by CC-RRT*.  As an extension to CC-RRT, distributionally robust RRT (DR-RRT) is introduced in~\cite{summers2018distributionally}, where  
 a moment-based ambiguity set is used unlike our algorithm. 
 The resulting deterministic constraint is similar to the one in CC-RRT* with the difference that it leads to a stronger constraint tightening. 
 On the contrary, our DR-RRT* uses CVaR constraints in addition to the Wasserstein ambiguity set, which inherently takes into account moment ambiguity, thereby providing an additional layer of robustness as mentioned in Appendix~\ref{appendix2}. 
 Furthermore, it is worth mentioning that most motion planning algorithms work only for a restricted set of problems. For example, in both CC-RRT* and DR-RRT, the region occupied by obstacles should be represented by a convex polytope performing an uncertain linear translation, while in both Risk-RRT*~\cite{chi2017risk} and Risk-Informed-RRT*~\cite{chi2018risk} the risk map is constructed as a grid by discretizing the state space.
On the contrary, our method does not impose such restrictions, allowing any obstacle of an arbitrary shape and motion as long as the loss can be constructed as a piecewise quadratic function.

\section{Application to Learning-Based Distributionally Robust Motion Control}~\label{sec:LBDRMC}

In addition to motion planning, our  DR-risk map can be used for motion control in risky environments. 
As the second application, we propose a learning-based motion control technique that limits the risk of collision in a distributionally robust way. 
In this case, our motion controller determines a control input that is robust against errors in learned information about the obstacles' movements.

We formulate the motion control problem as the following MPC problem with DR-risk constraints:
\begin{subequations}\label{DRMPC}
\begin{align}
\min_{\bold{u},\bold{x},\bold{y}}\quad& J(x_r(t),\bold{u}):= \sum_{k=0}^{K-1} c(y_k,u_k)+q(y_K)\\
\textnormal{s.t.}\quad &x_{k+1}=f(x_k,u_k)\label{DRMPCcons1}\\
&y_{k}=C x_k\label{DRMPCcons2}\\
&x_0=x_r(t)\label{DRMPCcons3}\\
&\mathcal{R}_{t,k}(y_k)\leq\delta\label{DRMPCcons4}\\
&x_k\in\mathcal{X} \label{DRMPCcons5}\\
&u_k\in\mathcal{U}\label{DRMPCcons6}
\end{align}
\end{subequations}
where $\bold{x} := (x_0, \ldots, x_K)$, $\bold{u} := (u_0, \ldots, u_{K-1})$,  $\bold{y}:= (y_0, \ldots, y_{K})$ are the robot's predicted  state, input and output trajectories over the prediction horizon $K$. The constraints \eqref{DRMPCcons1} and \eqref{DRMPCcons6} should be satisfied for $k=0,\dots,K-1$, the constraint \eqref{DRMPCcons2} should hold for $k=0,\dots,K$, and the constraints \eqref{DRMPCcons4} and \eqref{DRMPCcons5} are imposed for $k=1,\dots,K$. Here, the stage-wise cost function $c: \mathbb{R}^{n_y} \times \mathbb{R}^{n_u} \to \mathbb{R}$ and the terminal cost function $q: \mathbb{R}^{n_y} \to \mathbb{R}$ are chosen to penalize the deviation from the reference trajectory $y^{ref}$ and to minimize the control effort as follows:
\begin{align*}
c(y_k,u_k)&=\|Q(y_k-y^{ref}_k\|_2^2+\|R u_k\|_2^2\\
q(y_K)&=\|Q_f(y_K-y^{ref}_K)\|_2^2,
\end{align*}
where $Q,Q_f,R\succ 0$ are the state and control weight matrices.
The sets $\mathcal{X}$ and $\mathcal{U}$ represent the state and input constraint sets, respectively, which are assumed to be polyhedra for simplicity. 

The constraint~\eqref{DRMPCcons4} integrates the risk map into the controller synthesis by limiting the DR-risk~\eqref{DR_risk} to  user-specified tolerance level $\delta$. When  $f(x_k,u_k)$ is a linear   function, the DR-MPC problem can be reformulated into a bi-linear SDP by writing the risk constraint in the SDP form~\eqref{DR-CVaR}. However, solving such a problem is a computationally expensive task. 
To alleviate the computational issue, we proposed to approximate the DR-risk map by an NN that can be trained offline.

\subsection{Neural Network Approximation of DR-Risk Map} \label{sec:app}

\begin{figure}
\centering
\includegraphics[width=0.7\linewidth]{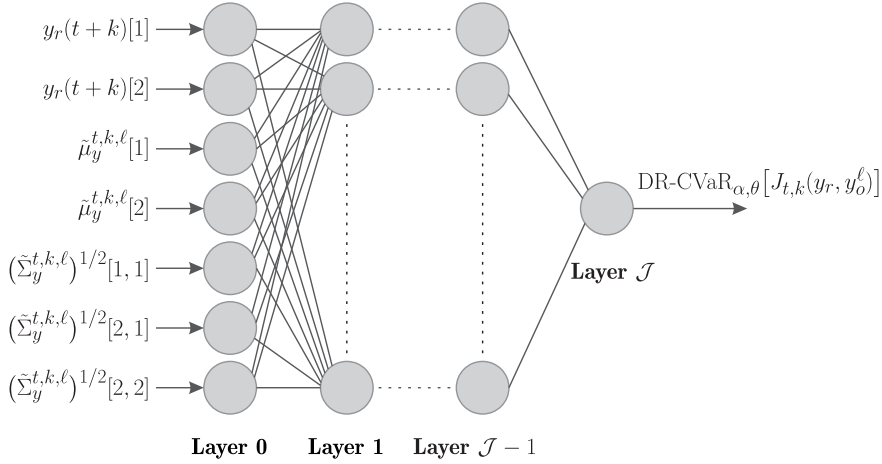}
\caption{Feed-forward NN for approximating the DR-risk map for fixed $\theta$ and $\alpha$.
The inputs are the robot's position $y_r$ and the parameters of the predicted distribution of the obstacles' behaviors $\tilde{\mu}^{t,k,\ell}_y$ and $\mathrm{vech}\big[(\tilde{\Sigma}^{t,k,\ell}_y)^{1/2}\big]$, while the target is the DR-risk. Here, $[i]$ refers to the $i$th entry of a vector, while $[i,j]$ refers to the entry in the $i$th row and the $j$th column of a matrix.}
\label{fig:NN}
\end{figure}

Consider the feed-forward NN in Fig.~\ref{fig:NN} with $\mathcal{J}$ layers and $\mathcal{N}_i$ nodes in each with a ReLU activation function. The inputs of the NN are the robot's position $y_r(t+k)$ and the parameters of the predicted distribution of the obstacles' behaviors, $\tilde{\mu}_y^{t,k,\ell}$ and $(\tilde{\Sigma}_y^{t,k,\ell})^{1/2}$, while the target is the solution of the SDP problem~\eqref{DR-CVaR}.
 
For any  position of the robot and the predicted position of the obstacle, the risk map computed in~\eqref{pos_risk}  can be approximated using the NN as
\begin{equation}\label{risk_NN}
\mathcal{R}^{t,k,\ell}_{NN}(y_r,\mathcal{Y}^{\ell}; \theta,\alpha) = \big(a^{k,\ell}_{\mathcal{J}}+r_{\ell}^2\big)^+,
\end{equation}
where
\begin{align}
&h^{k,\ell}_{i}=\max\{0,a^{k,\ell}_{i}\},\; i=1,\dots,\mathcal{J}-1\label{activ}\\
&a^{k,\ell}_i=W_i h^{k,\ell}_{i-1}+b_i,\; i=1,\dots,\mathcal{J}.\label{a_NN}
\end{align}
Here, $W_{i}\in\mathbb{R}^{\mathcal{N}_i\times \mathcal{N}_{i-1}}$ and $b_{i}\in\mathbb{R}^{\mathcal{N}_i}$ are the weight and bias,  $h^{k,\ell}_i\in\mathbb{R}^{\mathcal{N}_i}$ and $a^{k,\ell}_i\in\mathbb{R}^{\mathcal{N}_i}$ are the output and activation of the $i$th layer with $h^{k,\ell}_0\in\mathbb{R}^{\mathcal{N}_0}$ being the input of the network with $\mathcal{N}_0=n_y(n_y+5)/2$. The activation function in \eqref{activ} follows from the definition of ReLU. 
The input of the network is constructed from the robot's position $y_r(t+k)\in\mathbb{R}^{n_y}$ and the parameters of the predicted distribution of the obstacles' behaviors $\tilde{\mu}^{t,k,\ell}_y$ and $(\tilde{\Sigma}^{t,k,\ell}_y)^{1/2}$ as follows:
\[
h^{k,\ell}_0 = \big[y_r(t+k)^\top, (\tilde{\mu}^{t,k,\ell}_y)^\top, \mathrm{vech}\big[(\tilde{\Sigma}^{t,k,\ell}_y)^{1/2}\big]^\top\big]^\top,
\]
where $\mathrm{vech}[\cdot]$ is an operator vectorizing the lower triangular elements of the matrix.
 Note that the NN is independent of $t$, $k$ and $\ell$ since the dependence is encoded in the input information. 
Therefore, we can use the same NN to approximate the DR-risk maps for all $t$, $k$ and $\ell$. 

To train the NN, a dataset is created by solving~\eqref{DR-CVaR} for  different values of $y_r(t+k), \tilde{\mu}^{t,k,\ell}_y$ and $(\tilde{\Sigma}^{t,k,\ell}_y)^{1/2}$ for fixed $\theta$ and $\alpha$. Thereafter, the NN is trained via backpropagation to approximate the DR-risk map. 
As an example, the mean squared error (MSE) and mean average error (MAE) for all training, validation, and test samples are reported in Table~\ref{Table_NN}, showing that both errors are small.

\begin{table}[!t]
\caption{Mean squared error (MSE) and mean average error (MAE) for the NN approximation of the DR-risk map with 405,000 training, 45,000 validation, and 50,000 test  data points.}
\centering
\setlength{\tabcolsep}{0.1em} 
\begin{tabular}{>{\raggedright}m{0.5cm} | >{\centering}m{2.5cm}| M{3.2cm} M{3.2cm} M{3.2cm}}
\hline
\multicolumn{2}{l}{\bf Radius $\theta$} & $10^{-5}$ & $10^{-3}$ & $10^{-2}$\\
\hline\hline
\multirow{3}{*}{\rotatebox[origin=c]{90}{\centering\bf{MSE}}} & \bf{Train} &  $9.036\times 10^{-7}$ & $2.780 \times 10^{-6}$ & $2.909\times 10^{-6}$\\
\cline{2-5}
& \bf{Validation} & $9.710\times 10^{-7}$ & $2.994\times 10^{-6}$ & $2.569\times 10^{-6}$\\
\cline{2-5}
& \bf{Test} & $9.100\times 10^{-7}$ & $3.343\times 10^{-6}$ & $2.538\times 10^{-6}$ \\
\hline
\multirow{3}{*}{\rotatebox[origin=c]{90}{\centering\bf{MAE}}} & \bf{Train} &  $2.637 \times 10^{-4}$ & $4.449 \times 10^{-4}$ & $4.473 \times 10^{-4}$\\
\cline{2-5}
& \bf{Validation} & $2.808 \times 10^{-4}$ & $3.624 \times 10^{-4}$ & $2.756\times 10^{-4}$\\
\cline{2-5}
& \bf{Test} & $2.806 \times 10^{-4}$ & $3.866 \times 10^{-4}$ & $2.757\times 10^{-4}$\\
\hline
\end{tabular}
\label{Table_NN}
%\vspace{-0.1in}
\end{table}

To validate this approach, we compare the DR-risk map and its NN approximation computed using 50,000 random realizations of $y_k$, $\tilde{\mu}_y^{t,k,\ell}\sim\mathcal{U}[0,10]^2$ and $\tilde{\Sigma}_y^{t,k,\ell}\sim\mathcal{U}[0, 0.7]^3$ for $\theta = 10^{-5},10^{-4},10^{-3},10^{-2}$ and $\alpha=0.95$. We also randomly generate the radius $r_\ell\sim\mathcal{U}[0, 0.2]$ and  the risk tolerance level $\delta_{\ell}\sim\mathcal{U}[0,0.5r_{\ell}^2]$ to show the flexibility of our approximation method.
As shown in Table~\ref{Table_safety}, the probability that the approximate risk map reports safe events as unsafe is quite small. 
Furthermore, the approximate risk map is not so conservative since the probability of misreporting unsafe events as safe is also small. 
These results show the validity of our NN approximation approach.

  \begin{table}[!t]
\caption{Validation of the approximate risk map.}
\centering
\setlength{\tabcolsep}{0.1em} 
\begin{tabular}{>{\raggedright}m{2cm} | >{\centering}M{3.4cm}| M{3.4cm}}
\hline
\multirow{2}{*}{\bf Radius $\theta$} & Safe events reported as unsafe & Unsafe events reported as safe\\[0.1cm]
\cline{2-3}
\hline\hline
$10^{-5}$ & $1.5 \times 10^{-3}$ & $4.0  \times 10^{-3}$\\
\cline{1-3}
$10^{-4}$ & $1.4  \times 10^{-3}$ & $1.1  \times 10^{-3}$\\
\cline{1-3}
$10^{-3}$ & $1.3  \times 10^{-3}$ & $1.0 \times 10^{-3}$\\
\cline{1-3}
$10^{-2}$ & $1.2  \times 10^{-3}$ & $8.4 \times 10^{-4}$\\
\hline
\end{tabular}
\label{Table_safety}
\end{table}

\subsection{Approximate Distributionally Robust MPC}

Using the NN approximation of the DR-risk map,
 we eliminate the need to solve the optimization problem~\eqref{DR-CVaR}  in the constraints of the MPC problem~\eqref{DRMPC}. 
 Moreover, since only the inputs $y_r(t+k), \tilde{\mu}_y^{t,k,\ell}$ and $\tilde{\Sigma}_y^{t,k,\ell}$ of the NN depend on the $t, k$ and $\ell$, the same NN can be used for all time stages and obstacles, by simply providing  appropriate inputs to the NN. 
 Therefore, the use of our NN approximation significantly reduces the computational burden required to solve the MPC problem. 
More specifically, we obtain the following approximate MPC problem.

\begin{proposition}
Suppose that the NN approximation~\eqref{risk_NN} of the DR risk map is given for fixed parameters $\theta$ and $\alpha$.
If the risk map in \eqref{DRMPCcons4} is replaced with the NN approximation,
the DR-MPC problem \eqref{DRMPC} can be expressed as follows:
\begin{subequations}\label{A-DRMPC}
\begin{align}
\min \quad  & J(x_r(t),\bold{u}):=\sum_{k=0}^{K-1} c(y_k,u_k)+q(y_K)\\
\textnormal{s.t.}\quad &x_{k+1}=f(x_k,u_k)\label{ADRMPCcons1}\\
&y_{k}=C x_k\label{ADRMPCcons2}\\
&x_0=x_r(t)\label{ADRMPCcons3}\\
&h_0^{k,\ell} = \big[y_k^\top, (\tilde{\mu}_y^{t,k,\ell})^\top, \mathrm{vech}\big[(\tilde{\Sigma}_y^{t,k,\ell})^{1/2}\big]^\top\big]^\top\label{ADRMPCcons4}\\
&W_{\mathcal{J}} h_{\mathcal{J}-1}^{k,\ell}+b_{\mathcal{J}}+r_\ell^2\leq \delta\label{ADRMPCcons5}\\
&h_{i}^{k,\ell}=\lambda^{k,\ell}_{i}+W_{i}h_{i-1}^{k,\ell}+b_i\label{ADRMPCcons6}\\
& h_{i}^{k,\ell} \geq 0, \; \lambda^{k,\ell}_{i}\geq 0 \label{ADRMPCcons7}\\
&(\lambda^{k,\ell}_{i})^\top h_{i}^{k,\ell}=0\label{ADRMPCcons8}\\
&x_k\in\mathcal{X}\label{ADRMPCcons9}\\
&u_k\in\mathcal{U},\label{ADRMPCcons10}
\end{align}
\end{subequations}
where $W_i$ and $b_i$ are the weights and the bias for the $i$th layer. Constraints~ \eqref{ADRMPCcons5}--\eqref{ADRMPCcons8}
are imposed for $i=1,\dots,\mathcal{J}$.
\end{proposition}

\begin{proof}
Consider the feasible set for constraint \eqref{DRMPCcons4}:
\begin{align*}
FS^k_{true} := &\{y_k \in \mathbb{R}^{n_y} \mid \max_{\ell=1,\dots,L}\mathcal{R}^\ell_{t,k}(y_k,\mathcal{Y}^\ell)\leq \delta\} \\
=&\{y_k \in \mathbb{R}^{n_y} \mid \mathcal{R}^\ell_{t,k}(y_k,\mathcal{Y}^\ell)\leq \delta \; \forall \ell\}.
\end{align*}
Using the NN approximation~\eqref{risk_NN} of the risk map,   the feasible set can be approximated by
\begin{equation}
FS^k_{NN} := \{y_k\in \mathbb{R}^{n_y} \mid \mathcal{R}^{t,k,\ell}_{NN}(y_r,\mathcal{Y}^{\ell}; \theta,\alpha)\leq \delta \; \forall \ell\}.\label{app_feas}
\end{equation}
For fixed $i$, $k$ and $\ell$, the ReLU in \eqref{activ} can be interpreted as projecting $a_{i}$ onto the non-negative orthant, i.e.,
\begin{equation}\label{proj}
h_{i}=\argmin_{x\in\mathbb{R}^{\mathcal{N}_i}} \Big \{ \frac{1}{2}\|x-a_{i}\|_2^2 \mid x\geq 0 \Big \}.
\end{equation}
Since \eqref{proj} is a convex optimization problem, 
$h_{i}=x^*$ and $\lambda_i^*$ are its primal and dual optimal solutions if and only if 
  the following KKT conditions are satisfied:
\begin{equation}
\begin{split}
x^*&=\lambda^*_{i}+a_{i}\\
(\lambda^*_{i})^\top x^* &=0\\
\lambda_{i}^* & \geq 0\\
x^* &\geq 0.
\end{split}\label{KKTcond}
\end{equation}
Replacing constraint \eqref{DRMPCcons4} in the original 
MPC problem
with \eqref{app_feas} and then expressing ReLU \eqref{activ} as \eqref{KKTcond}, we obtain the approximate DR-MPC problem.
\end{proof}

The problem~\eqref{A-DRMPC} can be solved using nonlinear programming algorithms, such as interior-point methods~\cite{wright1997primal,nesterov1994interior,wachter2006implementation}, sequential quadratic programming~\cite{ferreau2014qpoases, powell1978fast,gill2005snopt,leineweber2003efficient}. 
Moreover,  it can also be solved using spatial branch-and-bound algorithms that exploit the bilinear nature of the nonconvex constraint. Similarly, branch-and-bound algorithms~\cite{liberti2008introduction, smith1999symbolic,ryoo1995global, quesada1995global} can be used replacing the nonlinear ineqalities~\eqref{ADRMPCcons4}--\eqref{ADRMPCcons7} with corresponding big-M constraints. 
In this work, for computational efficiency, we employ the interior-point solver implemented in FORCES Pro, which is tailored to efficiently find a local optimal solution for multistage optimization problems~\cite{zanelli2020forces}.

\section{Simulation Results}~\label{sec:result}

In this section, we provide two case studies to demonstrate the performance and utility of our DR-risk map: one for motion planning and another for motion control. All algorithms were implemented in MATLAB and run on a PC with a 3.70 GHz Intel Core i7-8700K processor and 32 GB RAM. The SDP problems~\eqref{DR-CVaR} and~\eqref{DR-CVaR_dual}   were solved using conic solver MOSEK~\cite{aps2019mosek}. In the motion control experiment  the FORCES Pro~\cite{zanelli2020forces} was used to solve the approximate DR-MPC problem.

\subsection{Motion Planning}
As with the first case study, motion planning is performed using our learning-based DR-RRT* in dynamic 2D environments. We consider a car-like robot with the following discrete-time kinematics:
\begin{align*}
\mathrm{x}_r(t+1)&=\mathrm{x}_r(t)+T_s v_r(t)\cos(\theta_r(t))\\
\mathrm{y}_r(t+1)&=\mathrm{y}_r(t)+T_s v_r (t) \sin(\theta_r(t))\\
\mathrm{\theta}_r(t+1)&=\mathrm{\theta}_r(t)+T_s \frac{v_r(t)}{L_r} \tan(\delta_{r}(t)),
\end{align*}
where $\mathrm{x}_r(t), \mathrm{y}_r(t)$ and $\theta_r(t)$ are the states of the vehicle---representing the Cartesian coordinates of the robot's CoM and its heading angle--- and the velocity $\mathrm{v}_r(t)$ and steering angle $\delta_r(t)$ are the control inputs at time step $t$.
The sampling time is
 $T_s=0.1\: sec$, and $L_r=0.8\: m$ is the length of the robot. 
 Note that the robot can be covered by a circle with radius $r_r = 1$.

We consider two different scenarios: $(i)$ a 2D environment with obstacles with unknown dynamics, and $(ii)$ a 2D environment with obstacles with single integrator dynamics. In both cases, the parameters for the risk map are chosen as $\alpha=0.95, r_s = 0.1$ and $r_o^{\ell}=1$ for all $\ell=1,2$, while the maximum depth for the tree is chosen as $K=10$. The control inputs for the robot are limited to $|v_r(t)|\leq 5\: m/v^2$ and $|\delta_r(t)|\leq 30\: deg$. In the beginning of the algorithm, since there are no observations, the GPR dataset $\mathcal{D}^{\ell}$ includes only the current values of the $\ell$th obstacle's states and inputs. New samples are added to the dataset as time goes on. 

\begin{figure}
     \centering
     \begin{subfigure}[b]{0.7\linewidth}
     \centering
     \includegraphics[width=\linewidth]{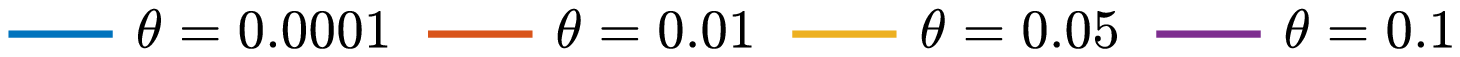}
     \end{subfigure}\\
     \begin{subfigure}[b]{0.2\linewidth}
         \centering
         \includegraphics[trim=0 15 107 0,clip, width=\linewidth]{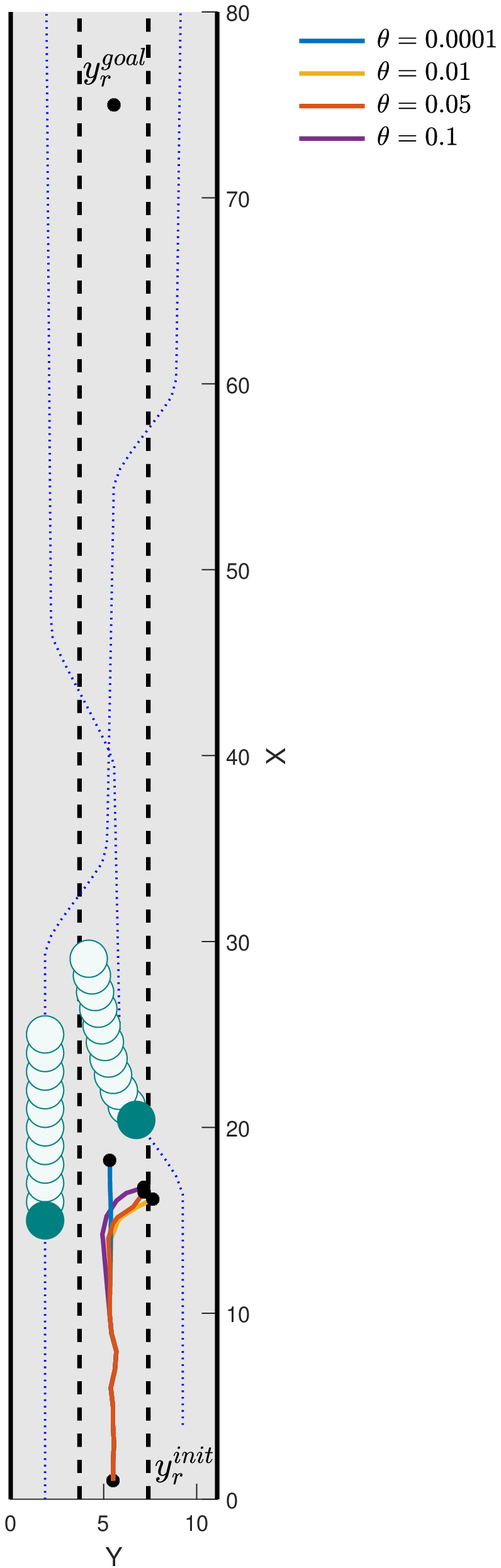}
         \caption{$t=19$}
         \label{fig:Sim1}
     \end{subfigure}%
     \begin{subfigure}[b]{0.2\linewidth}
         \centering
         \includegraphics[trim=0 15 107 0,clip,width=\linewidth]{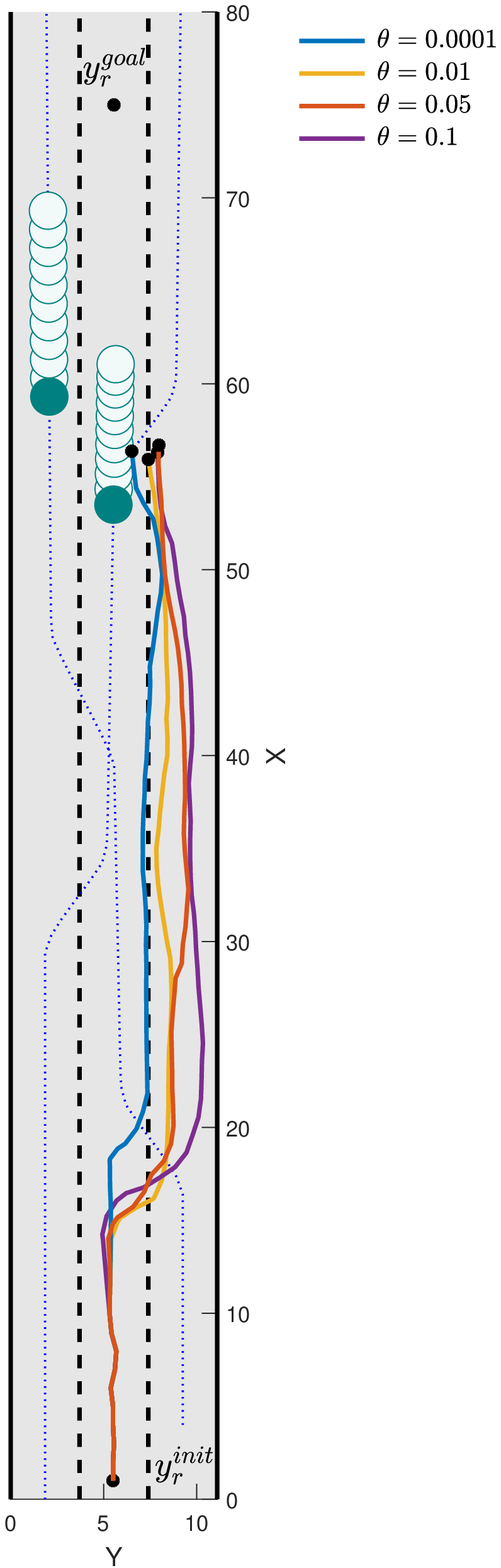}
         \caption{$t=60$}
         \label{fig:Sim2}
     \end{subfigure}%
     \begin{subfigure}[b]{0.2\linewidth}
         \centering
         \includegraphics[trim=0 15 107 0,clip,width=\linewidth]{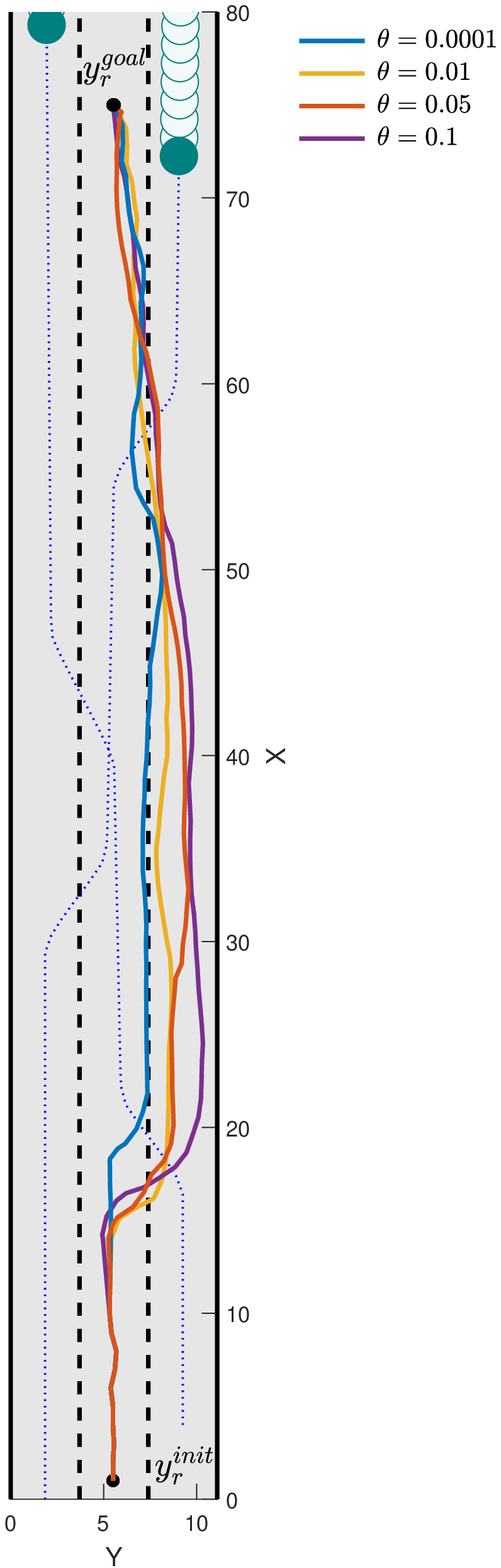}
         \caption{$t=80$}
         \label{fig:Sim3}
     \end{subfigure}
        \caption{Application of learning-based DR-RRT* to a car-like robot on a highway for $\theta=0.0001,0.01,0.05,0.1$. The obstacles are shown in green, while their predicted positions are shown in lighter color.}
        \label{fig:Sims}
\end{figure}

\subsubsection{Highway Scenario}
In the first scenario, the robotic vehicle navigates a highway-like 2D environment  with $L=2$ obstacles with unknown behaviors. We parameterize the dynamics model $\phi^{\ell}$ as described in Appendix~\ref{appendix1} using a previously obtained transition dataset of $10^5$ observations and a feedforward NN with $3$ hidden layers, $20$ neurons in each. The state for each obstacle consists of the Cartesian coordinates of its CoM and the heading angle, while the inputs are its velocity and angular acceleration.

Fig.~\ref{fig:Sims} shows the trajectories generated by learning-based DR-RRT* for  $\theta= 0.0001,0.01,0.05,0.1$ at different time instances, where two obstacles are shown in green. 
The goal point is on the second lane. For this experiment the risk tolerance level $\delta = 0.2205$ is set to be $5\%$ of the maximum possible risk $r_\ell^2 = (r_r + r_o^\ell +r_s)^2$. 
Fig.~\ref{fig:Sims}~(a) presents the situation when the first obstacle changes the lane from the third to the second lane. Since  the obstacle will be on the same lane as the robot according to the prediction, all paths generated by DR-RRT* except for $\theta=0.0001$ choose to move to the third lane. 
The case of $\theta = 0.0001$ is less conservative than the other cases as expected. 

\begin{figure}[t]
\centering
\includegraphics[width=0.3\linewidth, trim=2 117 23 460, angle=270, clip]{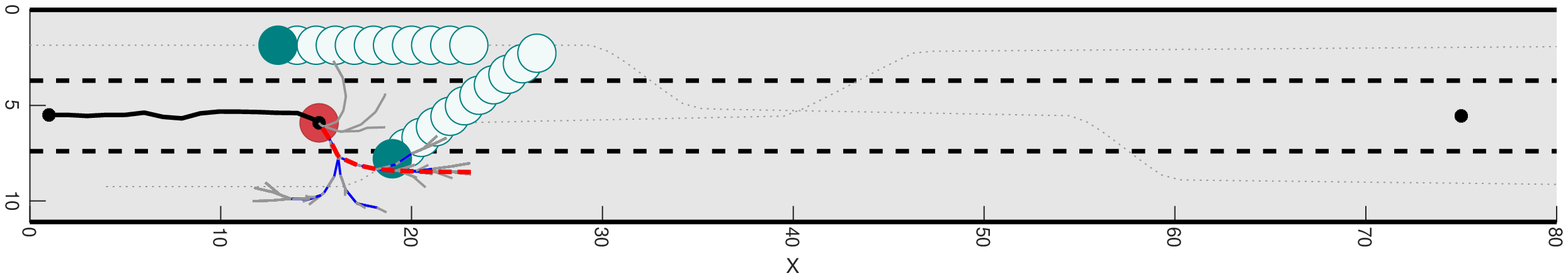}
\caption{Growing process of tree $\mathcal{T}$ (grey) and safe subtree $\mathcal{T}_{\mathrm{safe}}$ (blue) generation. The best path for execution (red) is chosen from $\mathcal{T}_{\mathrm{safe}}$.}
\label{fig:grow}
\end{figure}

After safely avoiding the obstacle, 
the robot needs to switch back to the second lane to reach the goal point. 
As shown in Fig.~\ref{fig:Sims}~(b),
the prediction of another obstacle's future motion indicates that the obstacle will continue following the second lane, while in reality it plans to move to the third lane. 
Since DR-RRT* with $\theta=0.0001$ considers errors in prediction only in a small ball, the robot chooses to overtake the obstacle, performing risky maneuvers. 
Meanwhile, the robot with a larger $\theta$ makes a safer decision, staying in the third lane. In Fig.~\ref{fig:Sims}~(c), all cases reach the desired goal point, completing the algorithm. Overall, it is observed that the case with the smallest radius $\theta=0.0001$ generates the most aggressive (but still safe) path. 
Increasing the radius drives the robot farther away from the obstacles, thereby guaranteeing a safe navigation with enough of a safety margin. 
Clearly, $\theta=0.1$ ensures a larger safety margin compared to the case of $0.01$ or $0.05$.

\begin{table}[!t]
\caption{The total operation cost and collision probability for the highway scenario.}
\centering
\setlength{\tabcolsep}{0.1em} 
\begin{tabular}{>{\raggedright}m{2.6cm}| M{2cm} M{2cm} M{2cm} M{2cm}}
\hline
{\bf Radius $\theta$} & $0.0001$ & $0.01$ & $0.05$ & $0.1$ \\
\hline\hline
\bf{Cumulative Cost} & $3222.64$ & $3224.32$ & $3302.63$ & $3796.14$ \\
\hline
\bf{Collision Probability}  & $0.052$ & $0.034$ & $0.028$ & $0.024$ \\
\hline
\end{tabular}
\label{Table1}
%\vspace{-0.1in}
\end{table}

 \begin{figure*}[t]
\centering
     \begin{subfigure}[b]{0.33\linewidth}
         \centering
         \includegraphics[width=\linewidth]{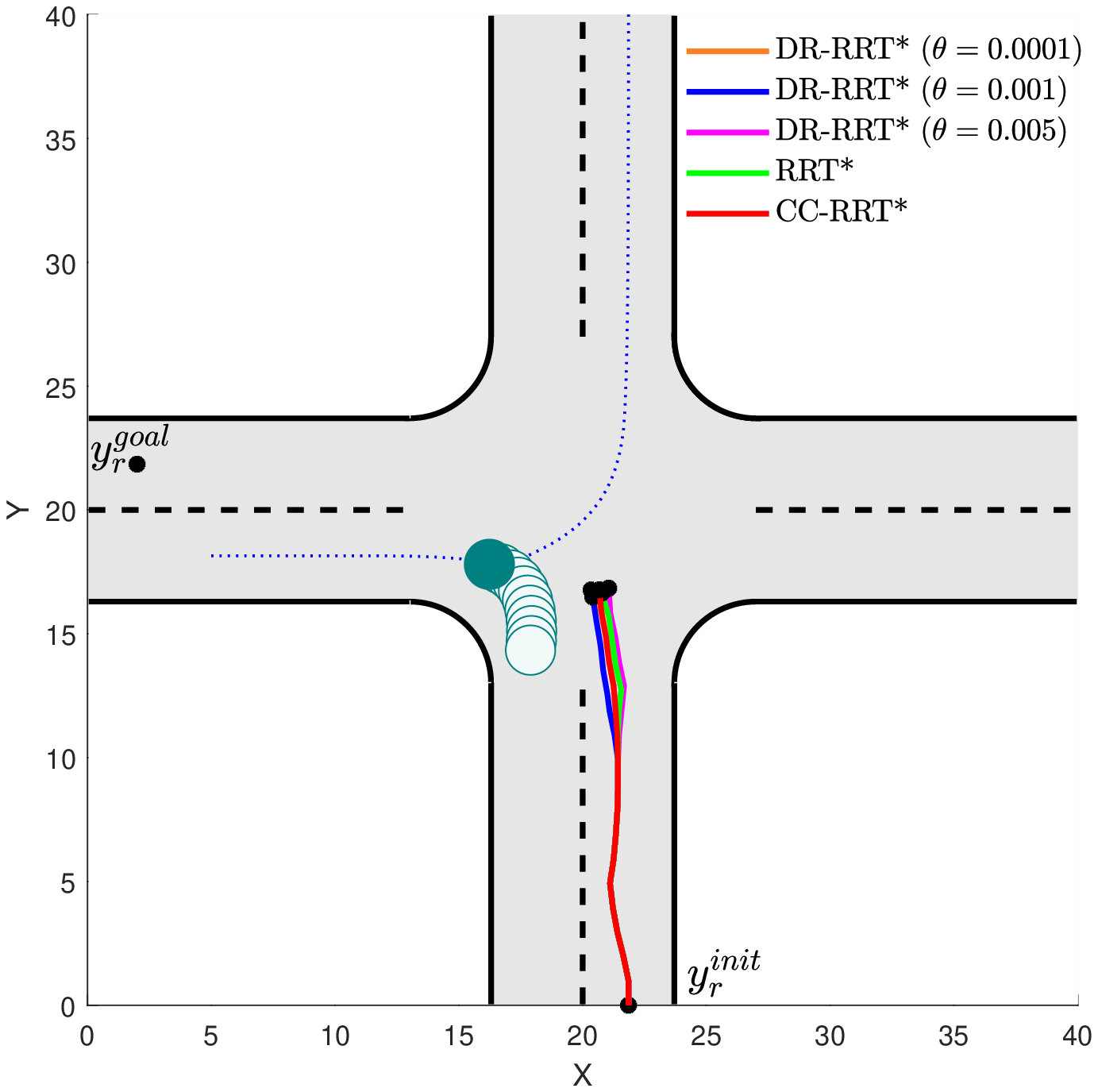}
         \caption{$t=18$}
         \label{fig:Sim_Sc2_1}
     \end{subfigure}%
     \begin{subfigure}[b]{0.33\linewidth}
         \centering
         \includegraphics[width=\linewidth]{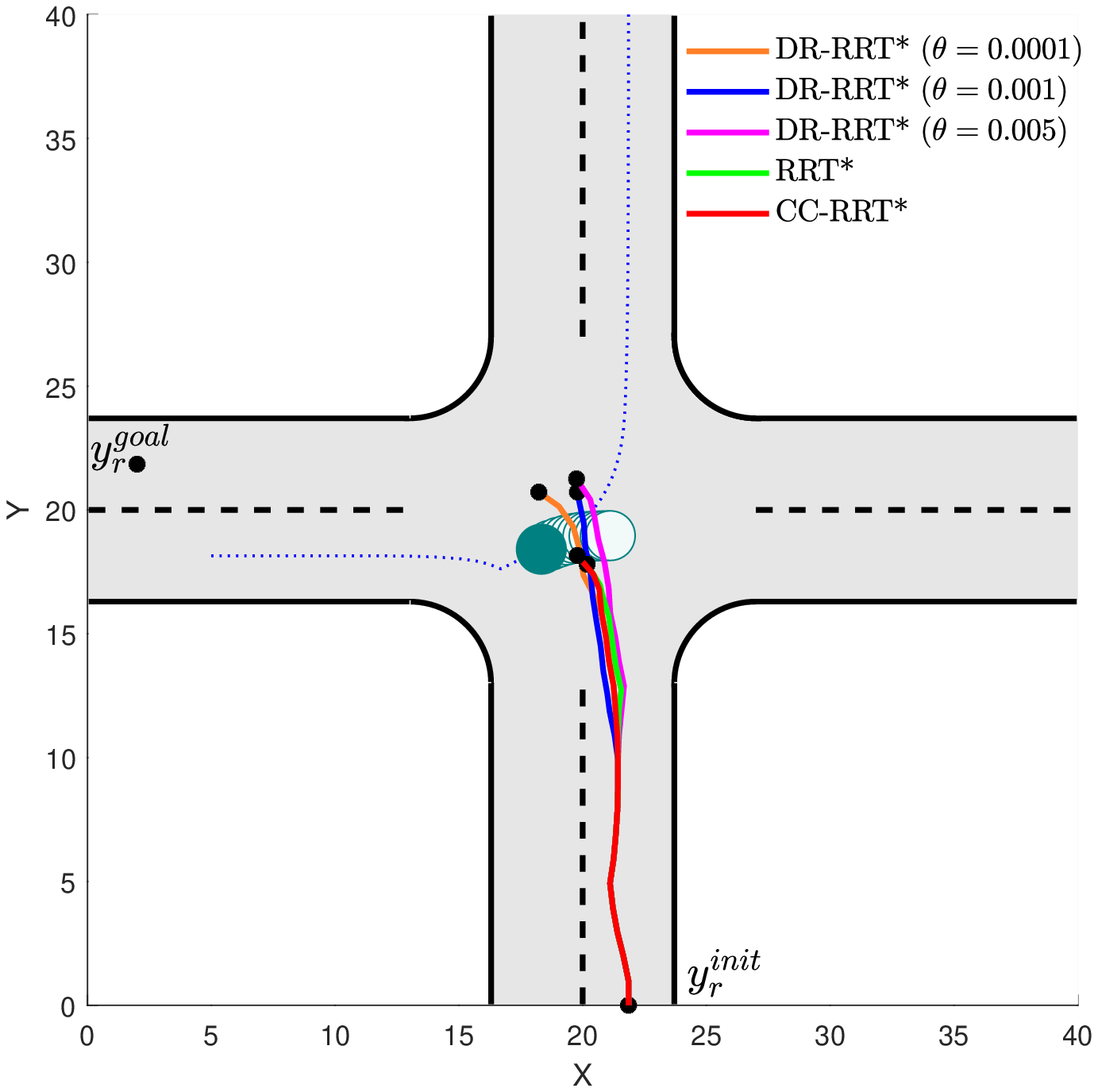}
         \caption{$t=23$}
         \label{fig:Sim_Sc2_2}
     \end{subfigure}%
     \begin{subfigure}[b]{0.33\linewidth}
         \centering
         \includegraphics[width=\linewidth]{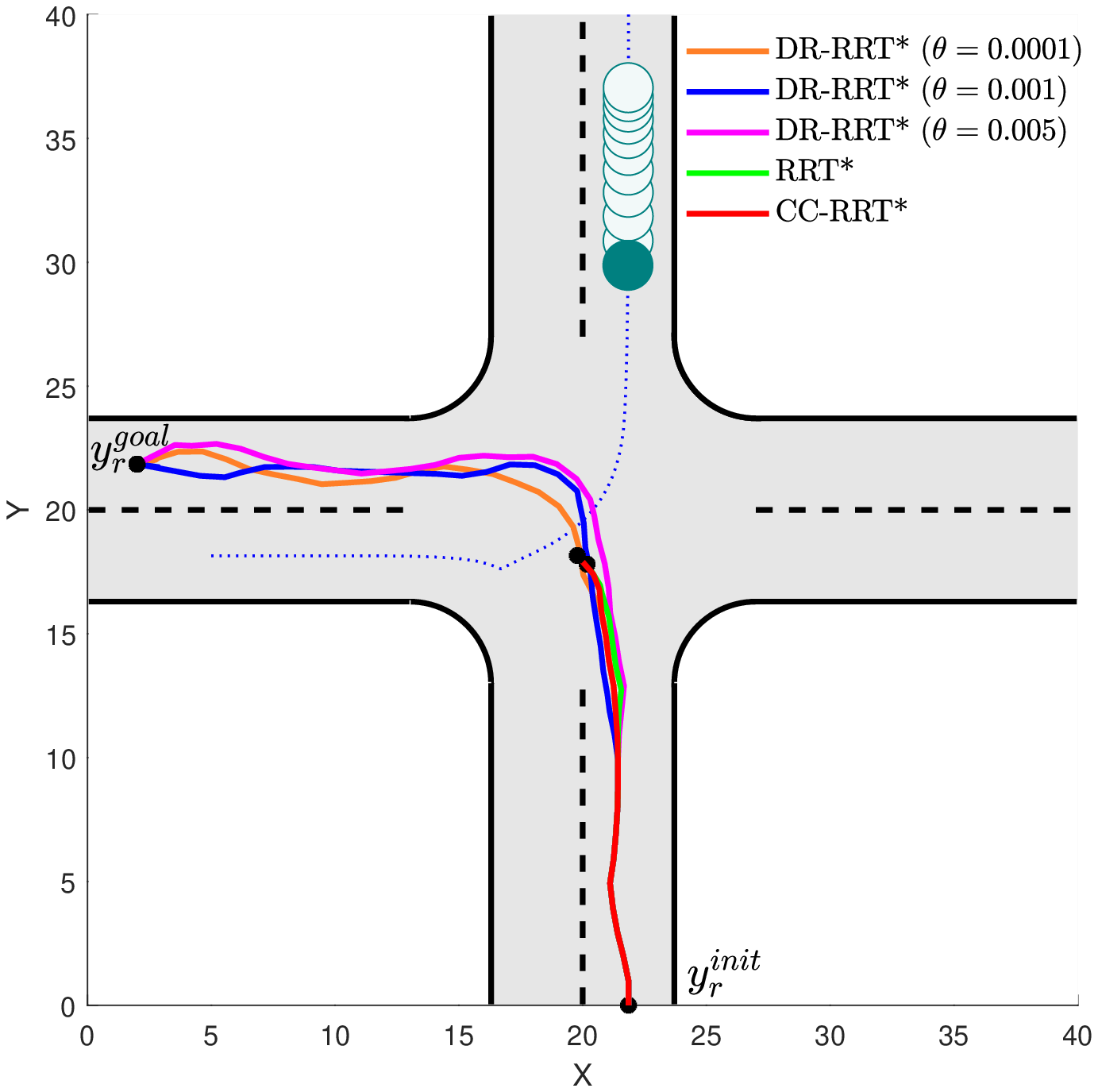}
         \caption{$t=43$}
         \label{fig:Sim_Sc2_3}
     \end{subfigure}
        \caption{Application of learning-based DR-RRT* to a car-like robot in an intersection for $\theta= 0.0001,0.001,0.005$ and comparison with RRT* and CC-RRT*. The obstacle is shown in green, while its predicted positions are shown in lighter color.}
        \label{fig:Sims_Sc2}
\end{figure*}

Fig.~\ref{fig:grow} illustrates how the tree grows at $t=18$ in the case of $\theta=0.01$. The tree starts from the current state of the robot. At the same time, GPR is executed to predict  the obstacles' future motions. Unfortunately, the prediction capability is poor when there are abrupt changes in the behavior of the obstacles. However, the prediction errors are taken into account in our DR-risk map, guaranteeing safety even when the prediction is not accurate. 
The grey tree corresponds to $\mathcal{T}$ obtained using Algorithm~\ref{DR_Risk_RRT}. 
However, to ensure safety, only the nodes with depth less than or equal to $K$ and satisfying the risk constraint are added to the safe subtree $\mathcal{T}_{\mathrm{safe}}$. The best path (in red) given to the robot for execution is then chosen from  $\mathcal{T}_{\mathrm{safe}}$.

Table.~\ref{Table1} shows the cumulative cost of the trajectories generated by DR-RRT* with different $\theta$'s. A bigger radius induces a more conservative behavior, driving the robot away from the shortest path.
Thus,  the total trajectory length and the cost increase with $\theta$.

To examine the robustness of our method, the average probability of collision is computed 
 by adding a random disturbance to the prediction result. 
Specifically, Gaussian noise $\hat{w}_{t,k,\ell}\sim\mathcal{N}(\mathbf{0},\hat{\Sigma}_w^{t,k,\ell})$ with covariance matrix $\hat{\Sigma}_w^{t,k,\ell}$ is added to the obstacles' positions, so that the predicted distribution is perturbed to  $\mathcal{N}(\tilde{\mu}_y^{t,k,\ell}+\hat{w}_{t,k,\ell},\tilde{\Sigma}_y^{t,k,\ell})$, where $\hat{\Sigma}_w^{t,k,\ell}=0.001 I_{n_y}$.
Thereafter, DR-RRT* is performed for $N = 1000$ different realizations of random variable $\hat{w}_{t,k,\ell}$.
The probability of collision is then calculated as the collision rate averaged over $N$ simulations, i.e.,
\[
\mathrm{P}_{\mathrm{coll}} = \frac{1}{N}\sum_{i=1}^N \mathrm{P}_{\mathrm{coll}}^{(i)},
\]
where the collision rate for each $i$th simulation run of length $T$ is computed by
\[
\mathrm{P}_{\mathrm{coll}}^{(i)} = \bigvee_{t=0}^{T}\bigvee_{k=0}^{K}\bigvee_{\ell=1}^{L}\mathbf{1}_{\{\|y_r^{(i)}(t+k)-y^{\ell}_o(t+k)\|<r_{\ell}\}}.
\]
Here, $\bigvee$ means ``Logical Or",
$y_r^{(i)}(t+k)$ is the robot's position at time $t+k$ planned at stage $t$ and $y_o^{\ell}(t+k)$ is the $\ell$th obstacle's actual position at $t+k$. 
The results of our robustness test are reported in Table~\ref{Table1}.
For all $\theta$’s, the collision probability is very small and decreases with the size of the ambiguity set.

 \subsubsection{Road Intersection Scenario}
 
In the second scenario, we consider a road intersection, where an obstacle has an unknown behavior with linear dynamics modeled as a single integrator:
\[
x_o(t+1)=\begin{bmatrix}1 & 0\\0 & 1\end{bmatrix}x_o(t)+\begin{bmatrix}T_s & 0\\0 & T_s\end{bmatrix}u_o (t),
\]
where $x_o(t)$ is the obstacle's position and $u_o (t)$ is the velocity vector in each direction.
This setting allows us to compare our method with other algorithms
that can only handle limited problem classes. 
Specifically, we compare our method to the classical RRT*~\cite{karaman2011sampling} as well as the CC-RRT* algorithm~\cite{luders2013robust}. This comparison is not possible in the first scenario where angular uncertainties are considered in addition to the placement uncertainties; CC-RRT* can only handle the latter. 
In the case of RRT*, we assume that the prediction results are accurate and consider the predicted mean to be the actual obstacle's position, ignoring uncertainties. In the case of CC-RRT*, the obstacle is over-approximated as an octagon, to attain its polytopic representation.
CC-RRT* uses chance constraints assuring that the probability of navigating in the safe set is greater than or equal to $\alpha$.
We set the risk weight in the cost~\eqref{cost} to $w=0$ to ensure the same conditions for all algorithms.

Fig.~\ref{fig:Sims_Sc2} shows the simulation results of DR-RRT* with $\theta= 0.0001,0.001,0.005$ and comparisons to RRT* and CC-RRT* at different time instances. In Fig.~\ref{fig:Sims_Sc2}~(a), the robot reaches the intersection without considering the obstacle, as it is still not interfering with the robot's path. The obstacle is trying to turn right, which is predicted well by GPR. 
However, as shown in Fig.~\ref{fig:Sims_Sc2}~(b), when the robot is trying to steer left,  the obstacle abruptly changes its decision to turn left. This situation is clearly not predicted well by GPR, and therefore RRT* and  CC-RRT* both fail to find a feasible solution. 
However, our DR-RRT* takes into account such an error in the learning result, guiding the robot to avoid a collision. 
Even though DR-RRT* succeeds in generating a collision-free path for all $\theta$'s, 
the path with smaller $\theta$ is riskier than that with a bigger one. With the biggest radius ($\theta=0.005$), the robot avoids the obstacle with a sufficient safety margin. 
Finally, Fig.~\ref{fig:Sims_Sc2}~(c) shows the completed paths generated by DR-RRT*, whereas both RRT* and CC-RRT* fail to complete their paths. 
We can conclude that RRT* is not suitable for motion planning in a highly uncertain environment, while CC-RRT* is applicable if the prediction results are accurate, as it does not consider distributional errors. However, our DR-RRT* is capable of performing safe path planning even with the existence of distributional errors in the learning results. 

Similar to the previous scenario, the probability of collision  is computed  using perturbed predictions with the same perturbation parameters.
Both RRT* and CC-RRT* fail to complete motion planning, and thus the probability of collision for is $1$ for both. In the case of our DR-RRT*, the collision probability is $0$, meaning that there is no collision for all Wasserstein ambiguity sets considered in this specific experiment.

\subsection{Motion Control}

In the second case study, we consider a motion control problem for a service robot in a cluttered environment such as a restaurant. 
The mobile robot is assumed to move according to the following double integrator dynamics:
\[
x_r(t+1)= \begin{bmatrix}1 & 0 & T_s & 0\\0 & 1 & 0 & T_s\\0 & 0 & 1 & 0\\0 & 0 & 0 & 1\end{bmatrix}x_r(t)+\begin{bmatrix}\frac{T_s^2}{2} & 0\\0 & \frac{T_s^2}{2}\\ T_s & 0\\ 0 & T_s\end{bmatrix}u_r(t),
\]
where $x_r(t)=(\mathrm{x}_r(t), \mathrm{y}_r(t), \mathrm{v}_{xr}(t),\mathrm{v}_{yr}(t)) \in \mathbb{R}^4$ is the robot's state  at time $t$, consisting of the Cartesian coordinates of its CoM and the corresponding velocity vector, 
and the input $u_r(t)= (\mathrm{a}_{xr}(t), \mathrm{a}_{yr}(t)) \in \mathbb{R}^2$ is 
chosen as the acceleration vector.
Again, $T_s$ denotes the sampling time, selected as $0.1\; sec$.

 \begin{figure*}[t]
\centering
     \begin{subfigure}[b]{0.35\linewidth}
         \centering
         \includegraphics[width=0.9\linewidth]{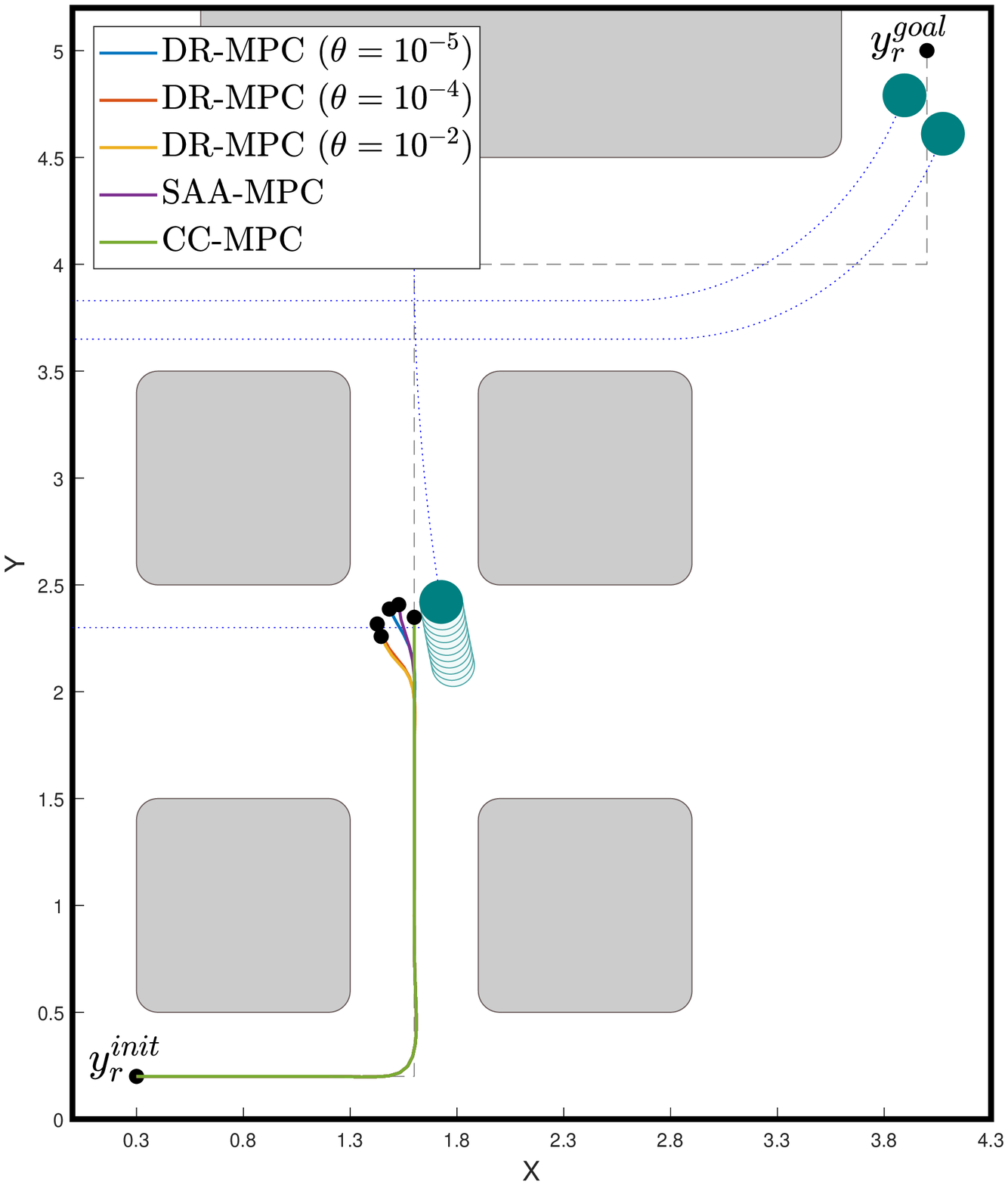}
         \caption{$t=55$}
         \label{fig:Sim_MPC_1}
     \end{subfigure}%
     \begin{subfigure}[b]{0.35\linewidth}
         \centering
         \includegraphics[width=0.9\linewidth]{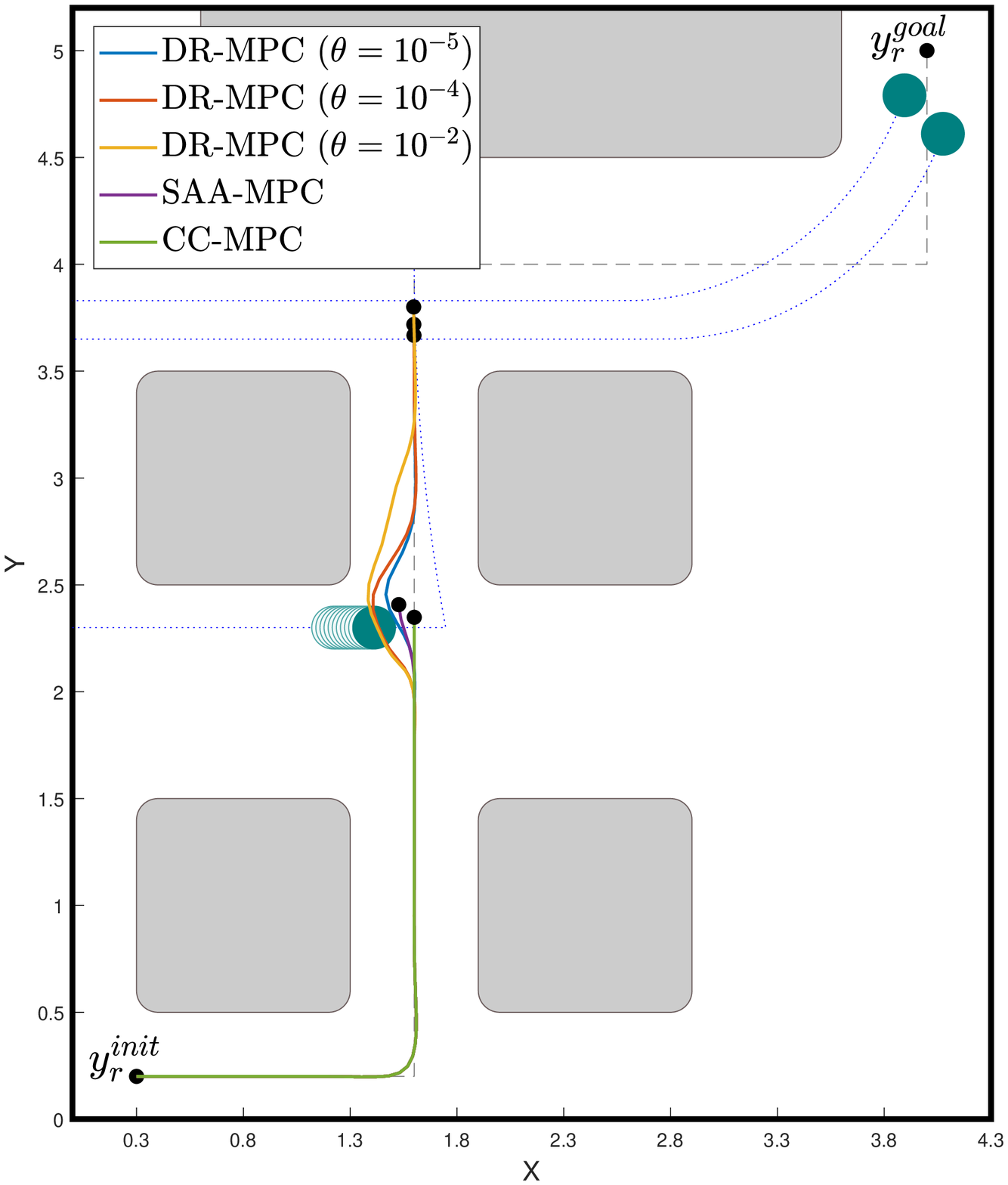}%
         \caption{$t=75$}
         \label{fig:Sim_MPC_2}
     \end{subfigure}\\
     \begin{subfigure}[b]{0.35\linewidth}
         \centering
         \includegraphics[width=0.9\linewidth]{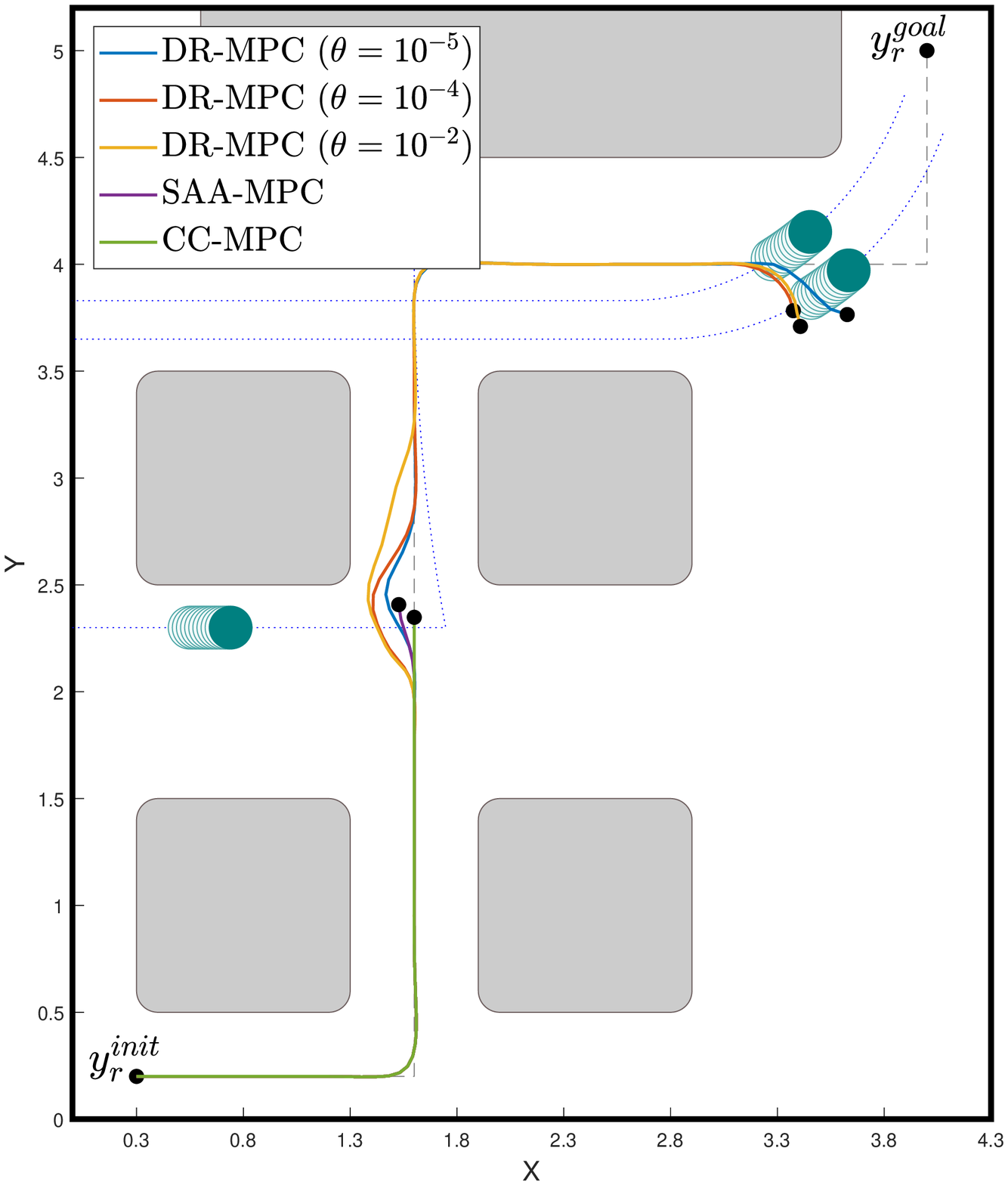}
         \caption{$t=110$}
         \label{fig:Sim_MPC_3}
     \end{subfigure}%
     \begin{subfigure}[b]{0.35\linewidth}
         \centering
         \includegraphics[width=0.9\linewidth]{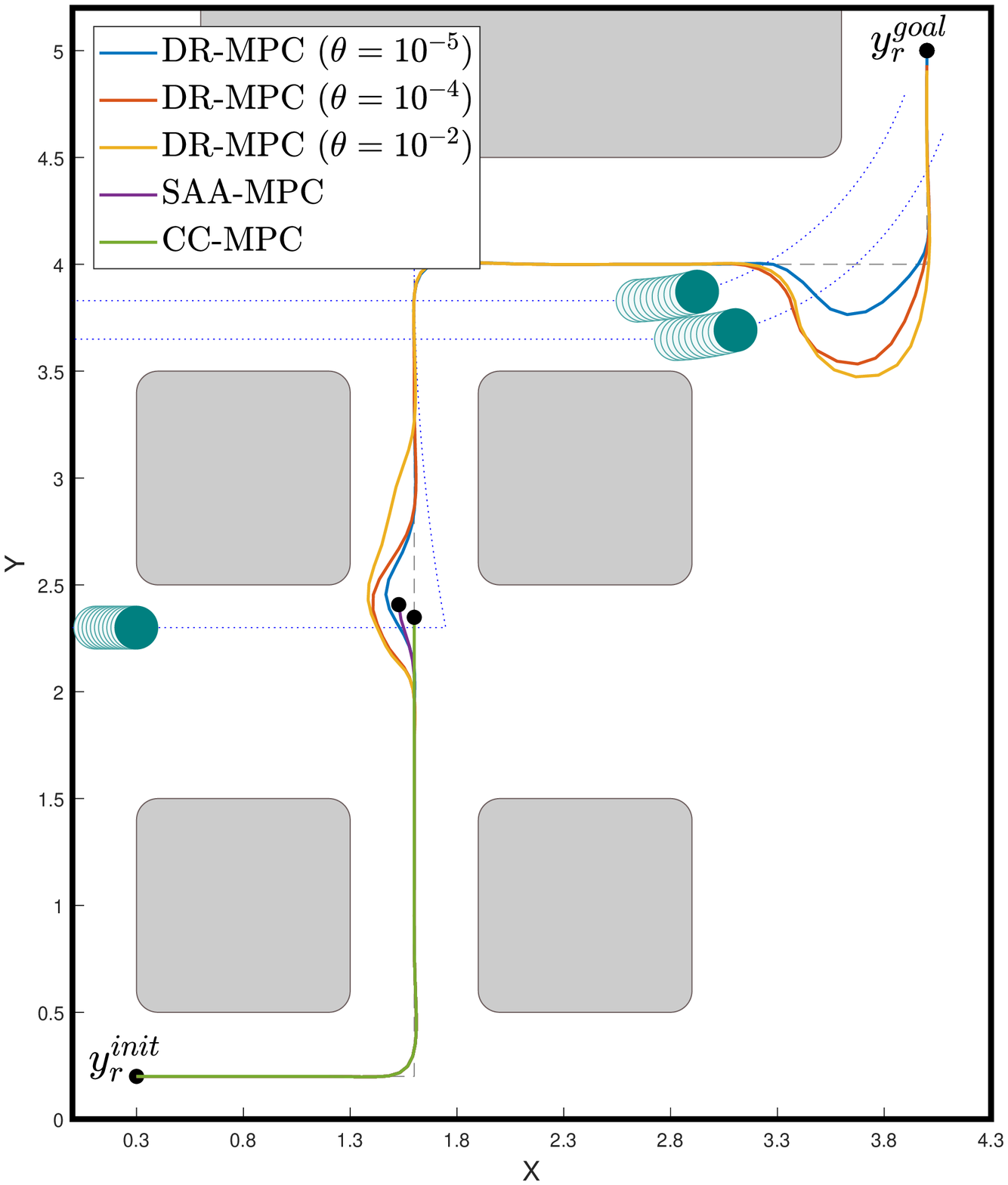}
         \caption{$t=133$}
         \label{fig:Sim_MPC_4}
     \end{subfigure}
        \caption{Application of learning-based DR-MPC to a car-like robot in a cluttered environment for $\theta = \{10^{-5}, 10^{-4}, 10^{-2}\}$, compared against CC-MPC and SAA-MPC with $N=100$. The obstacles are shown in green, while predictions for the corresponding obstacle are in lighter color.}
        \label{fig:Sims_MPC}
\end{figure*}

The circular robot of radius $r_r = 0.09$ aims to track a given reference trajectory in a cluttered 2D environment with some static and dynamic obstacles that may represent other service robots or human agents. Each of  $L=3$ dynamic obstacles is a circular object of radius $r_o^{\ell} = 0.1$, and the safety margin is set to be $r_s = 0.01$. 
The control input for the robot is limited to lie in $\mathcal{U}:=\{u \in \mathbb{R}^2 \mid \|u\|_{\infty} \leq 4\}$, while its state is restricted to $\mathcal{X}:=\{x \in \mathbb{R}^4 \mid (0,0,-2,-2) \leq x \leq (6,6,2,2)\}$. 
Each of the $L_\mathrm{stat}=5$ static obstacles is approximated by an ellipsoid, defined as $\mathcal{O}^i_\mathrm{stat}:=\{x\in\mathbb{R}^2\mid (x - x_\mathrm{stat}^i)^\top P_i^{-1} (x-x_\mathrm{stat}^i) \leq 1\}$, where $x_{\mathrm{stat}}^i$ is the center of $i$th elliptical obstacle and $P_i = P_i^\top \succ 0$ determines how far the ellipsoid extends in every direction from $x_\mathrm{stat}^i$. The following additional constraints are added to problem~\eqref{A-DRMPC} to avoid  the static obstacles:
\begin{align*}
(y_k - x_\mathrm{stat}^i)^\top P_i^{-1} (y_k - x_\mathrm{stat}^i)\geq 1 \quad  \forall i=1,\dots,L_\mathrm{stat}.
\end{align*}

The NN approximation of the DR-risk map is performed as described in Section~\ref{sec:app}. We uniformly sample 500,000 different values of $y_r(t+k)$ and $\tilde{\mu}_{y}^{t,k,\ell}$ from $\mathcal{U}[0,6]^2$ and $\mathrm{vech}\big[(\tilde{\Sigma}_y^{t,k,\ell})^{1/2}\big]$ from $\mathcal{U}[0,0.1]^3$ and divide them into training, validation and testing datasets with a ratio of $0.8:0.1:0.1$.

We begin the MPC algorithm by applying GPR to predict the mean $\tilde{\mu}_y^{t,k,\ell}$ and covariance $\tilde{\Sigma}_y^{t,k,\ell}$ for all dynamic obstacles $\ell=1,\dots,L$ for future time steps $k=1,\dots,K$ based on the latest $M=10$ observations of the obstacles' behaviors. 
This step is repeated in every time stage $t$ before solving the optimization problem~\eqref{A-DRMPC}.

We compare the performance of our approximate DR-MPC~\eqref{A-DRMPC} with that of the CVaR-constrained sample average approximation MPC (SAA-MPC)~\cite{hakobyan2019risk} with 
$N=100$ sample data generated from the predicted distribution, as well as the chance constrained MPC (CC-MPC) for elliptical obstacles~\cite{zhu2019chance}. 
The risk confidence level is chosen as $\alpha=0.95$.
For SAA-MPC and DR-MPC, the risk tolerance level $\delta = 4\times 10^{-4}$ is set to be $1\%$ of the maximum possible risk $r_\ell^2 = (r_r + r_o^\ell +r_s)^2$. In  our approximate DR-MPC, 
the radius  is chosen as $\theta = 10^{-5}, 10^{-4}, 10^{-2}$.

Fig.~\ref{fig:Sims_MPC} shows the simulation results for the three MPC methods with prediction horizon $K=10$. In both SAA-MPC and CC-MPC, the GPR prediction results are used for risk assessment. However, due to some sudden and unpredictable movements of the obstacles, the GPR results are not trustworthy.
As shown in Fig.~\ref{fig:Sims_MPC}~(a),  the mobile robot follows the reference trajectory and approaches the first dynamic obstacle. In this stage, all controllers try to avoid the obstacle by passing it on the left with different safety margins. 
However, even though CC-MPC finds a feasible solution under the inaccurately predicted distribution, collision occurs in reality due to the prediction and approximation errors.
Similarly, after a few steps, the robot controlled by SAA-MPC collides with the obstacle. 
Unlike the two controllers, DR-MPC controls the robot to safely avoid the obstacle and continue  following the reference trajectory despite the inaccurate GPR results. This is because, instead of directly using the learned distribution, DR-MPC considers the risk of unsafety with respect to the worst-case distribution within distance $\theta$ from the learned one.
This is shown in Fig.~\ref{fig:Sims_MPC}~(b), where the robot has already passed the obstacle. 
The radius $\theta$ affects the behavior of the robot in a way that increasing it results in a more risk-averse steering behavior. 
In particular, DR-MPC with $\theta=10^{-2}$ generates the most conservative trajectory, while the trajectory for $\theta=10^{-5}$ is the least safe, being close to that generated by SAA-MPC. 
This is because, as $\theta\to 0$, the ambiguity set vanishes and DR-CVaR reduces to CVaR. 
In Fig.~\ref{fig:Sims_MPC}~(c), the robot approaches the third and fourth dynamic obstacles. 
Similar to the previous situation, DR-MPC guides the robot to safely avoid the obstacles with some safety margins depending on the size of the ambiguity set. 
Finally, as shown in Fig.~\ref{fig:Sims_MPC}~(d),
the robot controlled by our DR-MPC method successfully reaches the goal point, 
unlike the other two methods.

The cumulative costs incurred by the three methods are reported in Table~\ref{Table2}.\footnote{In the cases of CC-MPC and SAA-MPC, we
continued to perform motion control even after collisions.}
Obviously, the cost  increases as the controller becomes more conservative,  as the robot drives away from the obstacles with larger safety margins.

Table~\ref{Table2} also shows
the probability of collision,  averaged over $500$ simulations by  perturbing the distribution predicted by GPR as in the motion planning case. 
CC-MPC has the highest probability of collision,
followed by SAA-MPC.
This is justified by the fact that chance constraint can be equivalently expressed using value-at-risk (VaR), while SAA-MPC uses CVaR. 
By definition, it holds that $\mathrm{VaR}[X]\leq\mathrm{CVaR}[X]$, and therefore the CVaR-based SAA-MPC induces more conservative behavior compared to CC-MPC. 
Our DR-MPC reduces the collision probability to $0.039$ even with a very small ambiguity set ($\theta=10^{-5}$).
Increasing the radius to $\theta=10^{-2}$ further reduces the probability of collision with the obstacles to $0.001$.

The computation time reported in Table~\ref{Table2} is measured from the starting point to the goal point. 
The results show  that 
CC-MPC and DR-MPC with $\theta=10^{-5}$ take a similar amount of time to complete motion control, while SAA-MPC is slightly slower due to the number of constraints in the optimization problem for each sample. 
As for the remaining $\theta$'s, increasing the safety of the robot is comparatively computationally heavy as finding a feasible trajectory satisfying the risk constraints becomes more time consuming. 
From these results, we can conclude that it is reasonable to use $\theta = 10^{-4}$ in this problem, which produces a sufficiently robust behavior  with moderate operation cost and computation time.

\begin{table}[!t]
\caption{Total operation cost, collision probability, and
total computation time   for CC-MPC, SAA-MPC, and DR-MPC.}
\centering
\setlength{\tabcolsep}{0.1em} 
\begin{tabular}{>{\raggedright}m{2.7cm}| M{2.3cm}| M{2.3cm}| M{2.3cm} M{2.3cm} M{2.3cm}}
\hline
&  \multirow{2}{*}{CC-MPC} &  \multirow{2}{*}{SAA-MPC} & \multicolumn{3}{c}{DR-MPC ($\theta$)}  \\
 \cline{4-6}
 & & & $10^{-5}$ & $10^{-4}$ & $10^{-2}$ \\
\hline\hline
\bf{Cumulative Cost} & $1.245$ & $3.665$ & $5.707$ & $18.430$ & $30.681$ \\
\hline
\bf{Collision Probability}  &$1$ & $0.056$ & $0.034$ & $0.005$ & $0.001$ \\
\hline
\bf{Computation Time} ($sec$) & $63.082$ & $71.513$ & $64.786$ & $69.494$ & $74.856$\\
\hline
\end{tabular}
\label{Table2}
\end{table}

\section{Conclusions}

We have proposed a novel risk assessment tool, called the DR-risk map, for a mobile robot in a cluttered environment with moving obstacles. 
Our risk map is robust against distribution errors in the obstacles' motions predicted by GPR.
For computational tractability, an SDP formulation was introduced along with its dual SDP. 
The utility of the risk map was demonstrated through its application to motion planning and control. 
The DR-RRT* algorithm uses the DR-risk map in the cost and constraint to generate a safe path in the presence of learning errors. 
Furthermore, to reduce the computational cost, 
 an NN approximation of the risk map was proposed and embedded into our MPC problem for motion control. 
The results of our simulation studies demonstrate the capability of the DR-risk map to preserve safety under learning errors.

It remains as future work to use the DR-risk map in other important problems such as safe learning for robotic systems and risk-sensitive reinforcement learning. 
Another interesting direction is to enhance the adaptivity of the DR-risk map by updating its conservativeness in an online manner depending on the observed safety margin.

\appendix
\section{Neural Network Approximation of Obstacle Dynamics}\label{appendix1}
As mentioned in Section~\ref{sec:MRO},  the system model of obstacles might be unknown in practice. However, with some observation data,
an approximate model $\phi_w$  of $\phi$ can be constructed using NNs. 
In this work, we use feedforward NNs with ReLU activation functions and $\mathcal{L}_{\phi}$ hidden layers to approximate the obstacles' dynamics.
The input of the NN consists of the obstacles' state and action vectors at each time stage.
The target of the NN is chosen as the difference between the next state and the current state to take advantage of the discrete nature of the dynamics.
The training dataset is collected through the observation of $N_{\phi}$ random transitions $(x_o(t),u_o(t),x_o(t+1))$,
\begin{align*}
D_\mathrm{in}^{\ell}=\big\{&(x_o(t),u_o(t))\big\}_{t=0}^{N_\phi-1}\\
D_\mathrm{tar}^{\ell}=\big\{&x_o(t+1)-x_o(t)\}_{t=0}^{N_\phi-1},
\end{align*}
where $D_\mathrm{in}^{\ell}$ and $D_\mathrm{in}^{\ell}$ represent the input and target datasets, respectively. 
Given the datasets, the NN $\hat{\phi}_w$ is trained  by minimizing the mean squared error:
\[
L_{\phi}(w) = \sum_{t=0}^{N_{\phi}-1} \frac{1}{2} \|\hat{\phi}_{w}(x_o(t),u_o(t)) - (x_o(t+1)-x_o(t))\|^2,
\]
where the parameter vector $w$ represents the network weights.
As a result of optimization, we obtain the following approximate model for obstacle dynamics:
\begin{equation}
\phi_w(x_o(t),u_o(t))=x_o(t)+\hat{\phi}_w(x_o(t),u_o(t)),\label{app_obs_model}
\end{equation}
which replaces the function $\phi$ in the obstacle dynamics \eqref{x_mu_sigma}.

\section{Proofs}\label{appendix2}

\subsection{Proof of Theorem~\ref{thm:sdp}}

\begin{proof}
We use the definition of CVaR to wrtie the DR-risk as follows:
\begin{align*}
\textrm{DR-CVaR}_{\alpha,\theta}\big[J(y_r,y_o)\big]
&=\sup_{\mathrm{Q}\in\mathbb{D}} \inf_{z\in\mathbb{R}}\bigg ( z+\frac{1}{1-\alpha}\mathbb{E}^{\mathrm{Q}}\big[\big(J (y_r,y_o ) -z\big)^+\big] \bigg )\\
&\leq \inf_{z\in\mathbb{R}} \bigg ( z+\frac{1}{1-\alpha}\sup_{\mathrm{Q}\in\mathbb{D}}\mathbb{E}^{\mathrm{Q}}\big[\big(J (y_r,y_o )-z\big)^+\big] \bigg ),  \nonumber
\end{align*}
where the inequality follows from the minimax inequality.

 Consider the following convex uncertianty set, which is the projection of $\mathbb{D}$ onto the space of means and covariances:
\begin{equation}
\begin{split}
\mathcal{U}_\theta(\tilde{\mu},\tilde{\Sigma})=\Big\{(\mu,\Sigma)&\in\mathbb{R}^{n_y}\times \mathbb{S}^{n_y}_+ \mid \|\mu-\tilde{\mu}\|_2^2+B^2(\Sigma,\tilde{\Sigma})\leq \theta^2\Big\}.
\end{split}\label{Utheta}
\end{equation}
The uncertainty set $\mathcal{U}_{\theta}(\tilde{\mu},\tilde{\Sigma})$ is convex and compact since it is the projection of the Wasserstein ball.
We now leverage the Gelbrich hull, defined in~\cite{kuhn2019wasserstein}, which contains all distributions supported on $\Xi$ whose mean and covariance fall into the uncertainty set $\mathcal{U}_\theta(\tilde{\mu},\tilde{\Sigma})$. 
In our case, since we consider two normal distributions, the Gelbrich hull is identical to the Wasserstein ball $\mathbb{D}$ defined in~\eqref{WBall}. Due to nonlinearity of covariance matrix in the underlying distribution, it is reasonable to perform change of variables and represent the uncertainty set $\mathcal{U}_{\theta}(\tilde{\mu},\tilde{\Sigma})$ by the second-order moment $M=\mathbb{E}[y_o y_o^\top] = \Sigma+\mu\mu^\top$. Then the new uncertainty set $\mathcal{V}_\theta(\tilde{\mu},\tilde{\Sigma})$ will be defined as:
\begin{equation} \nonumber
\begin{split}
&\mathcal{V}_{\theta}(\tilde{\mu},\tilde{\Sigma}) = \big\{(\mu,M)\in\mathbb{R}^{n_y} \times \mathbb{S}_+^{n_y} \mid (\mu,M-\mu\mu^{\top})\in \mathcal{U}_{\theta}(\tilde{\mu},\tilde{\Sigma}) \big\},
\end{split}\label{Vtheta}
\end{equation}
which is also a convex set.

Now, we use the fact that the Gelbrich hull or the 2-Wasserstein ball in our case can be expressed as the union of Chebyshev ambiguity sets with means and covariances in the uncertainty set~\eqref{Utheta}. Equivalently, using the uncertainty set~\eqref{Vtheta}, the Gelbrich hull can be viewed as the union of Chebyshev ambiguity sets with first- and second-order moments in the uncertainty set~\eqref{Vtheta}, i.e.,
\begin{equation*}
\begin{split}
\mathbb{D} &= \bigcup_{(\mu,\Sigma)\in\mathcal{U}_{\theta}(\tilde{\mu},\tilde{\Sigma})}  \mathcal{C} (\mathbb{R}^{n_y},\mu,\Sigma)\\
&= \bigcup_{(\mu,M)\in\mathcal{V}_{\theta}(\tilde{\mu},\tilde{\Sigma})}  \mathcal{C} (\mathbb{R}^{n_y},\mu,M-\mu\mu^\top),
\end{split}
\end{equation*}
where $\mathcal{C} (\mathbb{R}^{n_y},\mu,\Sigma)$ is the Chebyshev ambiguity set containing all distributions on $\mathbb{R}^{n_y}$ with mean $\mu$ and covariance bounded above by $\Sigma$. 
Thus, we have
\begin{equation*}
\begin{split}
&\sup_{\mathrm{Q}\in\mathbb{D}}\mathbb{E}^\mathrm{Q}\big[\big( J(y_r,y_o)-z\big)^+\big] =\sup_{(\mu,M)\in\mathcal{V}_{\theta}(\tilde{\mu},\tilde{\Sigma})} \sup_{\mathrm{Q}\in\mathcal{C}(\mathbb{R}^{n_y},\mu,M-\mu\mu^{\top})} \mathbb{E}^\mathrm{Q}\big[\big( J(y_r,y_o)-z\big)\big].
\end{split}
\end{equation*}

In the above equation the inner optimization problem measures the risk for all distributions with given first- and second-order moments, while the outer one considers the ambiguity in those moments with respect to the Wasserstein distance. Such two-layered optimization provides an addition robustness, accounting for moment ambiguities.

From~\cite[Lemma~A.1]{zymler2013distributionally} the inner supremum gets the following dual form:
\begin{align}
&\begin{cases}\inf  \; &\tau+2\gamma^\top \mu + \langle \Gamma, M \rangle\\
\textnormal{s.t. } & \tau+2\gamma^\top y_o+\langle \Gamma, y_o y_o^\top\rangle \geq J(y_r,y_o\big)-z\; \forall y_o\\
& \tau+2\gamma^\top y_o+\langle \Gamma, y_o y_o^\top\rangle \geq 0\; \forall y_o\\
& \tau\in \mathbb{R}, \gamma\in\mathbb{R}^{n_y}, \Gamma\in\mathbb{S}^{n_y}
\end{cases} \nonumber\\
&=\begin{cases}\inf  \; &\tau+2\gamma^\top \mu + \mathrm{Tr}[\Gamma M]\\
\textnormal{s.t.} & \begin{bmatrix}\Gamma+I & \gamma-y_r\\ (\gamma-y_r)^\top & \tau+ z+\|y_r\|_2^2 \end{bmatrix}\succeq 0\\
& \begin{bmatrix}\Gamma & \gamma\\ \gamma^\top & \tau\end{bmatrix}\succeq 0\\
& \tau\in \mathbb{R}, \gamma\in\mathbb{R}^{n_y}, \Gamma_+\in\mathbb{S}^{n_y},\end{cases}\label{Cheb}
\end{align}
where the second problem is obtained by replacing the quadratic constraint with the corresponding semidefinite one. By weak duality, the dual provides an upper bound of the inner supremum. Applying minimax inequality and replacing the inner supremum with its dual, we arrive at the following upper bound for the worst-case expectation:
\begin{equation}
\begin{split}
\inf_{\tau,\gamma,\Gamma} &\Big\{ \tau+\sup_{(\mu,M)\in\mathcal{V}_\theta(\tilde{\mu},\tilde{\Sigma})} \big(2\gamma^\top \mu+\mathrm{Tr}[\Gamma M] \big )  \mid \mbox{constraints in~\eqref{Cheb}} \Big\}.
\end{split}
\end{equation}
The inner supremum has an interesting form, which can be rewritten by the support function of $\mathcal{V}_{\theta}(\tilde{\mu},\tilde{\Sigma})$ evaluated at $(2\gamma,\Gamma)$. The support function $\sigma_{\mathcal{V}_{\theta}(\tilde{\mu},\tilde{\Sigma})}(q,Q)$ for any $q\in\mathbb{R}^m$ and $Q\in\mathbb{S}^m$ can found by solving the following SDP problem~\cite{kuhn2019wasserstein}:
\begin{align*}
\sigma_{\mathcal{V}_{\theta}(\tilde{\mu},\tilde{\Sigma})}(q,Q)=\inf_{\lambda,\varepsilon, Z} \; & \lambda(\theta^2-\|\tilde{\mu}\|_2^2-\mathrm{Tr}[\tilde{\Sigma}])+\varepsilon+\mathrm{Tr}[Z]\\
\textnormal{s.t.} \; & \begin{bmatrix}\lambda I-Q & \lambda\tilde{\mu}+\frac{q}{2}\\ \lambda \tilde{\mu}^\top+\frac{q^\top}{2} & \varepsilon\end{bmatrix}\succeq 0\\
& \begin{bmatrix}\lambda I-Q & \lambda\tilde{\Sigma}^{1/2} \\ \lambda\tilde{\Sigma}^{1/2} & Z\end{bmatrix}\succeq 0\\
& \lambda\in\mathbb{R}_+, \varepsilon\in\mathbb{R}_+, Z\in\mathbb{S}_+^m.
\end{align*}

The result of the theorem follows from replacing the support function with the corresponding SDP and plugging in the expression for the worst-case expectation back into DR-risk.
\end{proof}

\subsection{Proof of Corollary~\ref{cor:sdp}}

\begin{proof}
To derive the dual of \eqref{DR-CVaR}, we write the Lagrangian functions with multipliers $X,Y,W,V\succeq 0$ and $\eta,\beta \geq 0$ as
\begin{align*}
\mathscr{L}&=z+\frac{\Big[\tau+\varepsilon+\mathrm{Tr}[Z]+\lambda\big(\theta^2-\|\tilde{\mu}\|_2^2-\mathrm{Tr}[\tilde{\Sigma}\big] \big )\Big]}{1-\alpha}\\
&-\langle X_{11},\lambda I-\Gamma\rangle - 2 X_{12}^\top (\gamma +\lambda\tilde{\mu})- X_{22}\varepsilon-\langle Y_{11},\lambda I-\Gamma\rangle\\
&-2\langle Y_{12},\lambda \tilde{\Sigma}^{1/2} \rangle - \langle Y_{22},Z\rangle-\langle W_{11}, \Gamma+I\rangle - 2 W_{12}^\top (\gamma-y_r)\\
&- W_{22} (\tau+y_r^\top y_r +z)-\langle V_{11},\Gamma\rangle -2  V_{12}^\top \gamma -V_{22}\tau -\langle U, Z\rangle-\eta \lambda-\beta\varepsilon,
\end{align*}
where $X_{ij}$ is the $(i,j)$ entry of matrix $X$ and $\langle \cdot,\cdot\rangle$ is the matrix inner product. The dual function $g$ is obtained by minimizing the Lagrangian function with respect to the primal variables:
\begin{equation*}
\begin{split}
g &= -\mathrm{Tr}[W_{11}]-2W_{12}^\top y_r -W_{22}y_r^\top y_r +\min_z (1-W_{22})z\\
&+\min_{\lambda} \Big(\frac{\theta^2-\|\tilde{\mu}\|_2^2 - \mathrm{Tr}\big[\tilde{\Sigma}\big]}{1-\alpha}-\mathrm{Tr}\big[X_{11}+Y_{11}+2 Y_{12}^\top \tilde{\Sigma}^{1/2}\big] -2 X_{12}^\top \tilde{\mu}-\eta\Big)\lambda\\
& +\min_{\gamma} (-2 X_{12}-2W_{12}-2V_{12})^\top \gamma  + \min_\varepsilon \Big(\frac{1}{1-\alpha} - X_{22}-\beta\Big)\varepsilon +\min_\tau \Big(\frac{1}{1-\alpha} - W_{22} - V_{22}\Big)\tau\\
&+\min_{\Gamma} \langle X_{11}+Y_{11}-W_{11}-V_{11},\Gamma\rangle+\min_{Z} \langle I-Y_{22}-U,Z\rangle.
\end{split}
\end{equation*}
Finally, solving the inner minimization problems and maximizing the dual function $g$ with respect to the dual variables, we obtain the dual form \eqref{DR-CVaR_dual}.

\end{proof}

\bibliographystyle{IEEEtran}
\bibliography{reference}

\end{document}